\documentclass[sigconf]{aamas}

\usepackage{balance} 

\makeatletter
\gdef\@copyrightpermission{
  \begin{minipage}{0.2\columnwidth}
   \href{https://creativecommons.org/licenses/by/4.0/}{\includegraphics[width=0.90\textwidth]{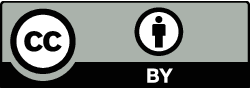}}
  \end{minipage}\hfill
  \begin{minipage}{0.8\columnwidth}
   \href{https://creativecommons.org/licenses/by/4.0/}{This work is licensed under a Creative Commons Attribution International 4.0 License.}
  \end{minipage}
  \vspace{5pt}
}
\makeatother

\setcopyright{ifaamas}
\acmConference[AAMAS '25]{Proc.\@ of the 24th International Conference
on Autonomous Agents and Multiagent Systems (AAMAS 2025)}{May 19 -- 23, 2025}
{Detroit, Michigan, USA}{Y.~Vorobeychik, S.~Das, A.~Nowé  (eds.)}
\copyrightyear{2025}
\acmYear{2025}
\acmDOI{}
\acmPrice{}
\acmISBN{}

\acmSubmissionID{915}

\title[Observation to State Diffusion]{On Diffusion Models for Multi-Agent Partial Observability: Shared Attractors, Error Bounds, and Composite Flow}

\author{Tonghan Wang\textsuperscript{*}}
\affiliation{
  \institution{Harvard University}
  \city{Cambridge, MA 02138}
  \country{USA}
}
\email{twang1@g.harvard.edu}

\author{Heng Dong\textsuperscript{*}}
\affiliation{
  \institution{Tsinghua University}
  \city{Beijing}
  \country{China}
}
\email{drdhxi@gmail.com}
\author{Yanchen Jiang}
\affiliation{
  \institution{Harvard University}
  \city{Cambridge, MA 02138}
  \country{USA}
  }
\email{yanchen_jiang@g.harvard.edu }
\author{David C. Parkes}
\affiliation{
  \institution{Harvard University}
    \city{Cambridge, MA 02138}
  \country{USA}
  }
\email{parkes@eecs.harvard.edu}
\author{Milind Tambe}
\affiliation{
  \institution{Harvard University}
    \city{Cambridge, MA 02138}
  \country{USA}
  }
\email{milind_tambe@harvard.edu }

\begin{abstract}
Multiagent systems grapple with partial observability (PO), and the decentralized POMDP (Dec-POMDP) model highlights the fundamental nature of this challenge. Whereas recent approaches to addressing PO have appealed to deep learning models, providing a rigorous understanding of how these models and their approximation errors affect agents' handling of PO and their interactions remain a challenge. In addressing this challenge, we investigate reconstructing global states from local action-observation histories in Dec-POMDPs using diffusion models. We first find that diffusion models conditioned on local history represent possible states as stable fixed points. In collectively observable (CO) Dec-POMDPs, individual diffusion models conditioned on agents' local histories share a unique fixed point corresponding to the global state, while in non-CO settings, shared fixed points yield a distribution of possible states given joint history. We further find that, with deep learning approximation errors, fixed points can deviate from true states and the deviation is negatively correlated to the Jacobian rank. Inspired by this low-rank property, we bound a deviation by constructing a surrogate linear regression model that approximates the local behavior of a diffusion model. With this bound, we propose a \emph{composite diffusion process} iterating over agents with theoretical convergence guarantees to the true state.
\end{abstract}


\newcommand{\BibTeX}{\rm B\kern-.05em{\sc i\kern-.025em b}\kern-.08em\TeX}

\usepackage{amsmath}
\usepackage{mathtools}
\usepackage{amsthm}
\usepackage{wrapfig}
\usepackage{thmtools}
\usepackage{nicefrac}

\usepackage[capitalize]{cleveref}
\crefname{section}{Sec.}{Secs.}
\Crefname{section}{Sec.}{Secs.}
\crefname{figure}{Fig.}{Figs.}
\Crefname{figure}{Fig.}{Figs.}
\crefname{appendix}{Appx.}{Apps.}
\Crefname{appendix}{Appx.}{Apps.}

\makeatletter
\def\thmheadbrackets#1#2#3{%
  \thmname{#1}\thmnumber{\@ifnotempty{#1}{ }\@upn{#2}}%
  \thmnote{ {\the\thm@notefont[#3]}}}
\makeatother

\newtheoremstyle{brakets}
  {}
  {}
  {\itshape}
  {}
  {\bfseries}
  {.}
  { }
  {\thmheadbrackets{#1}{#2}{#3}}

\newtheoremstyle{definitionbrakets}
  {}                       
  {}                       
  {\normalfont}               
  {}                          
  {\bfseries}                 
  {.}                         
  { }                         
  {\thmheadbrackets{#1}{#2}{#3}}

\theoremstyle{brakets}

\theoremstyle{definitionbrakets}
\newtheorem{finding}{Finding}
\newtheorem{evidence}{Evidence}[finding]
\newtheorem{definition}{Definition}

\definecolor{darkgreen}{rgb}{0.0, 0.5, 0.0}
\definecolor{darkblue}{rgb}{0.0, 0.5, 1.0}
\usepackage[textsize=tiny]{todonotes}

\newcount\Comments  
\newcommand{\kibitz}[2]{\ifnum\Comments=1{\color{#1}{#2}}\fi}

\newcount\CommentsAdd  
\newcommand{\kibitzAdd}[2]{\ifnum\CommentsAdd=1{\color{#1}{#2}}\fi}
\definecolor{english}{rgb}{0.0, 0.5, 0.0}
\definecolor{tw}{rgb}{0.0, 0.0, 0.5}

\usepackage{xcolor}
\usepackage{xparse}

\usepackage{amsmath,amsfonts,bm}

\def\eqref#1{equation~\ref{#1}}

\def\1{\bm{1}}

\def\va{{\bm{a}}}

\DeclareMathAlphabet{\mathsfit}{\encodingdefault}{\sfdefault}{m}{sl}
\SetMathAlphabet{\mathsfit}{bold}{\encodingdefault}{\sfdefault}{bx}{n}

\def\gF{{\mathcal{F}}}

\def\gS{{\mathcal{S}}}

\def\sN{{\mathbb{N}}}

\def\sP{{\mathbb{P}}}

\def\sR{{\mathbb{R}}}

\newcommand{\E}{\mathbb{E}}

\newcommand{\lr}{\alpha}

\DeclareMathOperator*{\argmax}{arg\,max}
\DeclareMathOperator*{\argmin}{arg\,min}

\DeclareMathOperator{\Tr}{Tr}

\newcommand{\shorte}{\textup{\texttt{=}}}

\newcommand{\ie}{\textit{i}.\textit{e}.}
\newcommand{\eg}{\textit{e}.\textit{g}.}

\NewDocumentCommand{\dm}{o o}{\ensuremath{
f_\theta
\IfValueT{#1}{\IfBlankTF{#1}{}{({#1}},}
\IfValueT{#2}{\IfBlankTF{#2}{}{{#2}})}
}}
\NewDocumentCommand{\optdm}{o o}{\ensuremath{
f^\star
\IfValueT{#1}{\IfBlankTF{#1}{}{({#1}},}
\IfValueT{#2}{\IfBlankTF{#2}{}{{#2}})}
}}
\NewDocumentCommand{\estpost}{o o}{\ensuremath{
p_\theta
\IfValueT{#1}{\IfBlankTF{#1}{}{({#1}}|}
\IfValueT{#2}{\IfBlankTF{#2}{}{{#2}})}
}}
\NewDocumentCommand{\localdata}{o o o o}{\ensuremath{
\mathcal{M}
\IfValueT{#1}{\IfBlankTF{#1}{}{_{#1}}}
\IfValueT{#2}{\IfBlankTF{#2}{}{^{#2}}}
\IfValueT{#3}{\IfBlankTF{#3}{}{(#3}}
\IfValueT{#4}{\IfBlankTF{#4}{}{,#4}}
\IfValueT{#3}{\IfBlankTF{#3}{}{)}}
}}
\NewDocumentCommand{\flow}{o o o o}{\ensuremath{
\phi
\IfValueT{#1}{\IfBlankTF{#1}{}{_{#1}}}
\IfValueT{#2}{\IfBlankTF{#2}{}{^{#2}}}
\IfValueT{#3}{\IfBlankTF{#3}{}{(#3}}
\IfValueT{#4}{\IfBlankTF{#4}{}{,#4}}
\IfValueT{#3}{\IfBlankTF{#3}{}{)}}
}}
\NewDocumentCommand{\fpset}{o o o o}{\ensuremath{
\mathcal{F}
\IfValueT{#1}{\IfBlankTF{#1}{}{_{#1}}}
\IfValueT{#2}{\IfBlankTF{#2}{}{^{#2}}}
\IfValueT{#3}{\IfBlankTF{#3}{}{(#3}}
\IfValueT{#4}{\IfBlankTF{#4}{}{,#4}}
\IfValueT{#3}{\IfBlankTF{#3}{}{)}}
}}
\NewDocumentCommand{\pusheq}{o o o o}{\ensuremath{
[\phi
\IfValueT{#1}{\IfBlankTF{#1}{}{_{#1}}}
\IfValueT{#2}{\IfBlankTF{#2}{}{^{#2}}}
\IfValueT{#3}{\IfBlankTF{#3}{}{(#3}}
\IfValueT{#4}{\IfBlankTF{#4}{}{,#4}}
\IfValueT{#3}{\IfBlankTF{#3}{}{)}}
]_*
}}
\NewDocumentCommand{\fpd}{o o o o}{\ensuremath{
D
\IfValueT{#1}{\IfBlankTF{#1}{}{_{#1}}}
\IfValueT{#2}{\IfBlankTF{#2}{}{^{#2}}}
\IfValueT{#3}{\IfBlankTF{#3}{}{(#3}}
\IfValueT{#4}{\IfBlankTF{#4}{}{,#4}}
\IfValueT{#3}{\IfBlankTF{#3}{}{)}}
}}
\NewDocumentCommand{\eigenvalue}{o o o o}{\ensuremath{
\lambda
\IfValueT{#1}{\IfBlankTF{#1}{}{_{#1}}}
\IfValueT{#2}{\IfBlankTF{#2}{}{^{#2}}}
\IfValueT{#3}{\IfBlankTF{#3}{}{(#3}}
\IfValueT{#4}{\IfBlankTF{#4}{}{|#4}}
\IfValueT{#3}{\IfBlankTF{#3}{}{)}}
}}
\NewDocumentCommand{\jaco}{o o o o}{\ensuremath{
J
\IfValueT{#1}{\IfBlankTF{#1}{}{_{#1}}}
\IfValueT{#2}{\IfBlankTF{#2}{}{^{#2}}}
\IfValueT{#3}{\IfBlankTF{#3}{}{(#3}}
\IfValueT{#4}{\IfBlankTF{#4}{}{|#4}}
\IfValueT{#3}{\IfBlankTF{#3}{}{)}}
}}
\NewDocumentCommand{\eigenvector}{o o o o}{\ensuremath{
e
\IfValueT{#1}{\IfBlankTF{#1}{}{_{#1}}}
\IfValueT{#2}{\IfBlankTF{#2}{}{^{#2}}}
\IfValueT{#3}{\IfBlankTF{#3}{}{(#3}}
\IfValueT{#4}{\IfBlankTF{#4}{}{|#4}}
\IfValueT{#3}{\IfBlankTF{#3}{}{)}}
}}
\NewDocumentCommand{\evas}{o o o o}{\ensuremath{
\Lambda
\IfValueT{#1}{\IfBlankTF{#1}{}{_{#1}}}
\IfValueT{#2}{\IfBlankTF{#2}{}{^{#2}}}
\IfValueT{#3}{\IfBlankTF{#3}{}{(#3}}
\IfValueT{#4}{\IfBlankTF{#4}{}{|#4}}
\IfValueT{#3}{\IfBlankTF{#3}{}{)}}
}}
\NewDocumentCommand{\eves}{o o o o}{\ensuremath{
U
\IfValueT{#1}{\IfBlankTF{#1}{}{_{#1}}}
\IfValueT{#2}{\IfBlankTF{#2}{}{^{#2}}}
\IfValueT{#3}{\IfBlankTF{#3}{}{(#3}}
\IfValueT{#4}{\IfBlankTF{#4}{}{|#4}}
\IfValueT{#3}{\IfBlankTF{#3}{}{)}}
}}
\newcommand{\ite}[2]{\ensuremath{#1^{(#2)}}}
\newcommand{\num}[2]{\ensuremath{#1^{(#2)}}}

\newcommand{\spname}[1]{\ensuremath{\mathtt{#1}}}

\usepackage{multirow}
\usepackage{makecell}
\newcolumntype{L}{>{$}l<{$}}
\newcolumntype{C}{>{$}c<{$}}
\newcolumntype{R}{>{$}r<{$}}

\begin{document}


\pagestyle{fancy}
\fancyhead{}


\maketitle 


\section{Introduction}


Given that the ability of individual agents to perceive complete information about the global state is limited~\citep{omidshafiei2017deep, srinivasan2018actor, saldi2019approximate, amato2013decentralized}, partial observability (PO) fundamentally characterizes the dynamics and interactions in multi-agent systems. Decentralized POMDPs (Dec-POMDPs)~\cite{oliehoek2016concise} highlight this information limitation, where many complex challenges are rooted in this issue, such as communication~\citep{foerster2016learning, wang2019learning}, decentralized control~\citep{de2006decentralized, yang2008multi}, cooperation\citep{wang2020roma, wen2022multi}, and coordination~\citep{xu2020learning, zhang2020bi} under incomplete information.

Decades of research addressing PO in the context of these challenges~\citep{kaelbling1998planning, loch1998using, goldman2003optimizing,spaan2005perseus} have fostered the development of specialized sub-fields~\citep{hausknecht2015deep, wang2020r, foerster2016learning, han2019variational, peng2017multiagent, zhang2021model, kao2022common} within the multi-agent system, thereby shaping its current landscape. As a general way of handling the uncertainty due to PO, the concept of belief states is introduced to represent an agent's probabilistic state estimation based on local information\cite{varakantham2006winning, muglich2022generalized, macdermed2013point}. While these methods effectively encapsulate uncertainty in some environments, traditionally, they may suffer from scalability issues due to the exponential growth in complexity of belief updates. Recent works use powerful deep learning models to address scalability, e.g., by directly predicting unseen state features~\cite{muglich2022generalized,xu2021side,jiang2018graph,xu2024beyond}. However, a rigorous understanding of how deep learning models and their approximation errors can impact agents' handling of PO and their interactions remains elusive -- a gap we address in this paper using diffusion models.






Diffusion models~\citep{ho2020denoising, kadkhodaie2023generalization,liu2023flow,song2021score,song2019generative} offer a novel promising avenue towards addressing uncertainty in Dec-POMDPs due to PO, specifically by learning the mapping from local histories to global states.  The primary challenge in learning such mappings lies in its inherent stochasticity and problem scale. A single history may correspond to multiple possible states, resulting in a one-to-many, stochastic mapping. Diffusion models, with their ability to model stochastic processes through the iterative denoising inductive bias, offer new opportunities to address such stochasticity. Additionally, the spaces of histories and states are often continuous and high-dimensional, also making diffusion models well-suited due to their proven powerful representational capacity in such expansive spaces~\cite{lou2024discrete,esser2024scaling,wu2024animating,guo2024animatediff,ho2022video,ceylan2023pix2video,rombach2022high,ho2022imagen}.

In this paper, we conduct an in-depth investigation into the use of diffusion models to manage PO in Dec-POMDPs, offering theoretical understandings supported by empirical evidences to solve the new challenges in this effort. To meet the requirements of decentralized control, our study comprises two steps. First, each agent infers states using a diffusion model conditioned on its local history. In this phase, we address critical problems, including how a diffusion model represents multiple possible states given local history, how accurate this representation is when the denoiser network is over- and under-parameterized, and methods to quantify the agent's uncertainty regarding the state. For the second step, should uncertainty persist, we study how to resolve it and optimally determine the true state by merging diffusion processes of all agents.

Specifically, our contributions are as follows. The first contribution is about how diffusion models represent states. In scenarios with minimal deep learning approximation errors, for each state $s$ consistent with local history $\tau$, the diffusion model conditioned on $\tau$ learns to create a \emph{stable fixed point} at the location $s$. Then, the repeated application of the \textcolor{black}{denoiser network} induces a discrete-time flow that transports noisy inputs to these attractors.

When agent $i$'s diffusion model conditioned on $\tau_i$ has a single fixed point, it can confidently infer the global state as this unique fixed point.
Our second contribution relates to complex scenarios with multiple fixed points. We establish that in collectively observable Dec-POMDPs, there exists a unique fixed point shared by all agents, corresponding precisely to the true global state.
Moreover, in non-collectively observable Dec-POMDPs, where aggregating local information cannot fully reveal the true state, shared fixed points represent all possible states given the joint history, and diffusion models can reproduce their true posterior probabilities. 

We then consider the influence of deep learning approximation errors typical in learning with large Dec-POMDPs. We find that the major impact is that agents' fixed points might deviate from the true state. Our third contribution is to investigate the underlying causes of these deviations and propose a method to bound their norm. Our theoretical analyses and empirical evidence suggest that deviations are inversely correlated with the Jacobian rank of the denoiser network at fixed points. This low-rank behavior enables us to approximate local behavior of diffusion models by a surrogate linear regression model, whose solution gives an upper bound to deviations. Empirical evidence supports the tightness of this bound.


The deviation of fixed points from states implies that it becomes impractical to determine the global state by intersecting the fixed points, as deviations vary among agents. To solve this problem, our fourth contribution is to propose the concept of \textit{composite diffusion} which denoises the input iteratively using \textcolor{black}{denoiser networks} conditioned on each agent's history. Theoretically, we prove that composite diffusion, regardless of the agents' order, converges to the true state with an error no larger than the deviation upper bound. We support the analyses by showing that composite diffusion leads to accurate global state estimation in the complex SMACv2 \citep{ellis2023smacv2} benchmark across a variety of highly stochastic testing cases. 

By providing a rigorous understanding of impacts of diffusion models on PO in Dec-POMDPs, this paper opens the door to newer algorithms for various multi-agent problems such as policy learning, coordination, and communication in complex environments.

\textbf{Related Works}.\ \ Diffusion models~\citep{ho2020denoising} have been extensively explored in single-agent settings, significantly advancing areas such as planning that require approximators of MDP dynamics~\citep{brehmer2024edgi, janner2022planning, liang2023adaptdiffuser}, data synthesis for reinforcement learning~\citep{chen2023genaug, lu2024synthetic}, and policy training on offline datasets~\citep{ada2024diffusion, chi2023diffusion, hansen2023idql}. In contrast, how to synergize diffusion models with multi-agent systems remains largely underexplored. \citet{xu2024beyond} investigate diffusion models within Dec-POMDPs but do not focus on the inherent stochasticity in mapping local histories to states, nor do they resolve how to address the disagreements among agents regarding the true state. Similarly, these critical problems remain untouched in other deep learning approaches that attempt to learn low-dimensional state representations~\citep{xu2021side, jiang2018graph} using variational autoencoder~\citep{doersch2016vae} or contrastive learning~\citep{chen2020simple}.

Orthogonal to our focus on studying the explicit reconstruction of states from local history, many sub-fields in multi-agent systems have developed innovative approaches to mitigate the effects of PO, such as modeling other agents~\citep{raileanu2018modeling, zhang2021multi}, intention inference~\cite{kim2020communication, han2018learning, qu2020intention}, and communication~\citep{sukhbaatar2016learning, jiang2018learning,kang2022non} for better decision-making. In multi-agent reinforcement learning, it is popular to employ RNNs~\citep{hausknecht2015deep,wang2020r,rashid2018qmix} or Transformers~\cite{wen2022multi,dong2022low} to process sequential history data, while incorporating global information during centralized training by techniques like value function decomposition~\citep{rashid2018qmix, sunehag2018value, wang2020rode} and global gradient approximation~\cite{yu2022surprising, wang2020dop}.

\section{Diffusion Models for Dec-POMDPs}

A Dec-POMDP~\cite{bernstein2002complexity,oliehoek2016concise} is a tuple $G\shorte\langle \mathcal C, \mathcal S, \mathcal A, P, R, \Omega, O, n, \gamma\rangle$, where $\mathcal A$ is the finite action set, $\mathcal C$ is the finite set of $n$ agents, $\gamma\in[0, 1)$ is the discount factor, and $s\in \mathcal S\textcolor{black}{\subseteq}\mathbb{R}^{|s|}$ is the true state. $|s|$ is the dimension of $s$. We consider partially observable settings and agent $i$ only has access to an observation $o_i\in \Omega$ drawn according to the observation function $O(s, i)$. Each agent has a history $\tau_i\in \mathcal{T}\equiv(\Omega\times \mathcal A)^*\times\Omega$. At each timestep, each agent $i$ selects an action $a_i\in \mathcal A$, forming a joint action $\va$ $\in \mathcal A^n$, leading to the next state $s'$ according to the transition function $P(s'|s, \va)$ and a shared reward $R(s,\va)$ for each agent. The joint history of all agents is denoted by $\tau_{1:n}$, or $\bm\tau$ when the agent order is irrelevant.

When referring to local history without specifying that it pertains to a particular agent, we employ the notation $\tau$. The mapping from $\tau$ to the corresponding global state $s$ is a one-to-many mapping due to partial observability. We formalize this mapping as follows.
\begin{definition}[History-State Mapping]
$S_G:\mathcal{T}\rightarrow 2^\mathcal{S}$ maps $\tau_i$ to the set of all possible states when agent $i$ observes $\tau_i$: $S_G(\tau_i) = \{s\mid p(s|\tau_i)>0\}$, where $p(s|\tau_i)$ is the posterior state distribution. We say that the states in $S_G(\tau_i)$ are \emph{consistent} with $\tau_i$.
\end{definition}

We use diffusion models to learn the mapping $S_G$ and reproduce the posterior $p(s|\tau_i)$.


\noindent \textbf{Diffusion models and scores}.\ \  Given a training dataset $\mathcal{D}=\{(\num{\tau_i}{k}, \num{s}{k})\}_{k=1}^K$, where $\num{s}{k}\in S_G(\num{\tau_i}{k})$ is a state consistent with local history $\num{\tau_i}{k}$, \textcolor{black}{a score-based model} $\dm:\mathcal{T}\times \mathbb{R}^{|s|} \rightarrow \mathbb{R}^{|s|}$ (also called a \emph{denoiser network}) is trained to minimize 
\begin{align}
    MSE(f_\theta,\sigma) = \mathbb{E}_{\tau_i,s,y\sim s+z}\left[\|s-\dm[\tau_i][y]\|^2\right],
\end{align}
Here $y=s+z$, where $z\sim\mathcal{N}(0,\sigma^2I)$. We refer to $s$ as clean states and $y$ as noisy states. During training, noisy states are generated by injecting a randomly sampled noise with a noise level $\sigma>0$ to $s$. The training involves the histories of all agents and the corresponding states. Our analyses in this paper are applicable to most neural network architectures, while in our experiments, we employ a simple fully-connected network with $\tau_i$ and $y$ as inputs and denoised state as output (details in Appendix~\ref{appx:emp}). This network is shared among all the agents.
\begin{figure*}
    \vspace{-2em}
    \centering
    \includegraphics[width=0.9\linewidth]{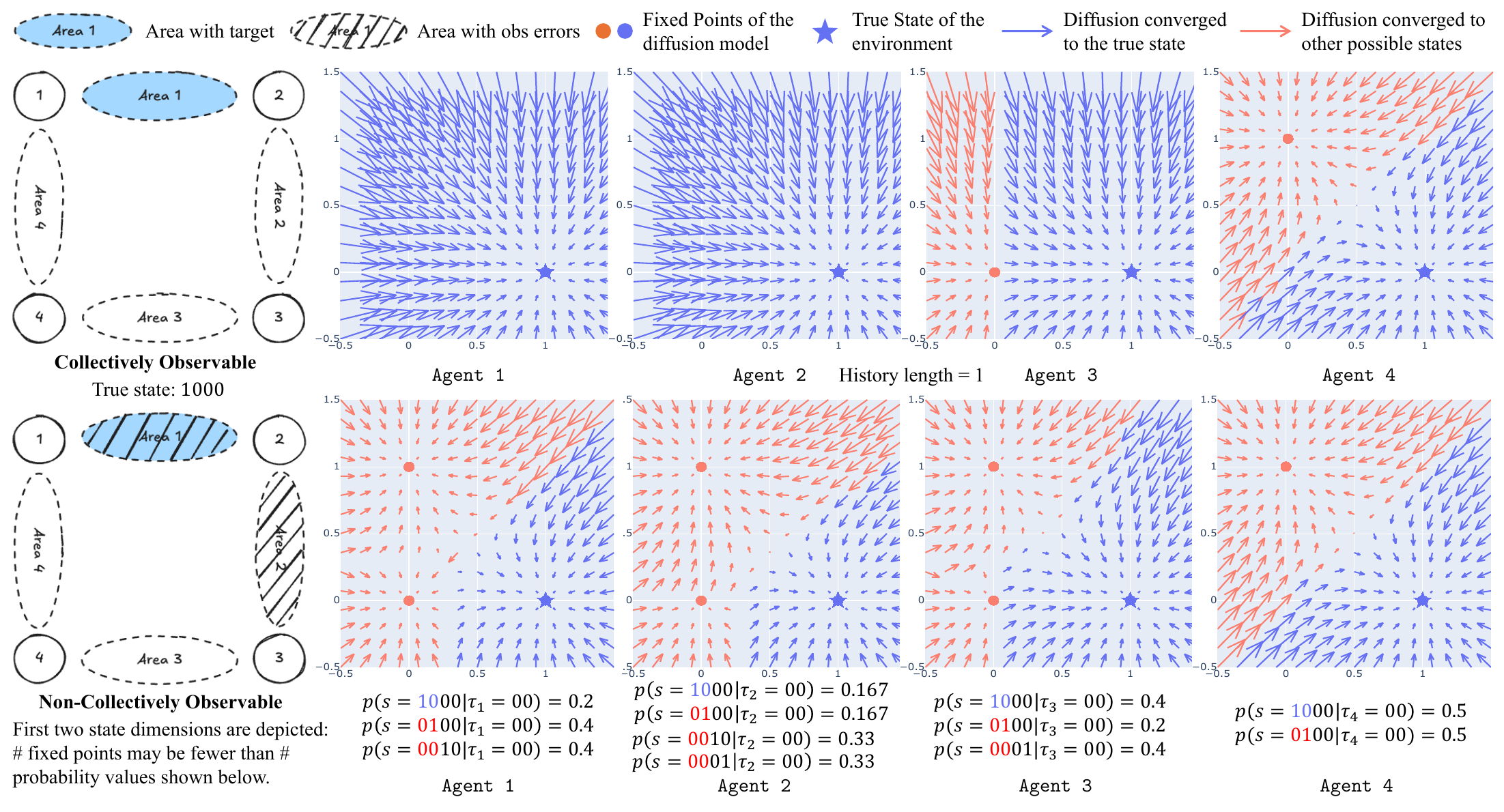}
    \vspace{-1.5em}
    \caption{\emph{With minimal deep learning approximation errors, a diffusion process represents states consistent with local history $\tau_i$ (length=1 in this figure) as their attractors, provably equivalent to stable fixed points of the \textcolor{black}{denoiser network} $\dm[\tau_i][\cdot]$.} Arrows point to \textcolor{black}{denoiser network} outputs $y'=\dm[\tau_i][y]$ from input noisy states $y$. The first two dimensions of $y'$ and $y$ are shown. Top row: In collectively observable (CO) Dec-POMDPs, a unique fixed point is shared by all agents, which is also the true state. Bottom row: In non-CO Dec-POMDPs, shared fixed points are all states consistent with joint history $\bm\tau$, and diffusion models reproduce the posterior state distribution $p(s|\tau_i)$ under appropriate distributions of input noisy states. }
    \label{fig:R0}
    \vspace{-1em}
\end{figure*}

As shown in~\cref{appx:math-miyasawa}, adapting the derivation from \citet{robbins1992empirical, miyasawa1961empirical, kadkhodaie2023generalization}, the optimal \textcolor{black}{denoiser network} yields the expected state given the noisy input $y$:
\begin{align}
    \optdm[\tau_i][y]=\mathbb{E}_s\left[s\big| y,\tau_i\right],\label{equ:f*}
\end{align}
which is related to the conditional scores by
    \begin{align}
     \nabla \log p_\sigma(y|\tau_i) =\frac{1}{\sigma^2}\left(\mathbb{E}_s\left[s\big| y,\tau_i\right]-y\right).\label{equ:f*_score}
\end{align}

\noindent \textbf{Posterior state distribution and discrete-time flow}.\ \ We are concerned with the estimated states and their distribution given history $\tau_i$. We diffuse states consistent with $\tau_i$ to noise by a diffusion process characterized by the "variance-exploding" stochastic differential equation (SDE)~\cite{song2021score}:
\begin{align}
dy = g(t)d\mathbf{w},\ \ g(t)=\sqrt{\frac{d[\sigma^2(t)]}{dt}},\label{equ:forward_sde}
\end{align}
where $\mathbf{w}$ is the standard Wiener process, \ie, Brownian motion. Let $\sigma(t)=Ae^{t}$, where $A$ is a constant, $t\in [0,T]$, and $T$ is the maximum timestep. According to Eq.~\ref{equ:forward_sde}, we have $g^2(t)=2\sigma^2(t)$. We generate states from noise by (reverse time) probability flow ordinary differential equation (ODE) conditioned on $\tau_i$:
\begin{align}
    dy = -\sigma^2(t)\nabla_y \log p_t(y|\tau_i) dt,\label{equ:ode}
\end{align}
which has the same marginal probability densities $\{p_t(y|\tau_i)\}_{t=0}^T$ as the time-reversal of the diffusing SDE in Eq.~\ref{equ:forward_sde}~\cite{song2021score}. Here $dt$ represents an infinitesimal negative time step. We approximate the solution to Eq.~\ref{equ:ode} by iteration
\begin{align}
    y(t-1)=y(t)+\sigma^2(t)\nabla_{y(t)} \log p_t(y(t)|\tau_i)=\optdm[\tau_i][y(t)].\label{equ:opt_ite}
\end{align}
The second equality follows from Eq.~\ref{equ:f*},~\ref{equ:f*_score}.

Rigorously, states are generated by applying a numerical solver to Eq.~\ref{equ:ode}, and Eq.~\ref{equ:opt_ite} brings discretization errors. To justify the use of this discretization, we show that we can still find the support of $p(s|\tau_i)$ (Thm.~\ref{thm:converge}) and that errors of $p(s|\tau_i)$ can be bounded (Thm.~\ref{thm:R0NonCO}). 

In practice, we approximate the iteration in Eq.~\ref{equ:opt_ite} by 
\begin{align}
    \ite{y}{\ell+1}=\dm[\tau_i][\ite{y}{\ell}],
\end{align}
where $\ell$ is the iteration index, with increasing $\ell$ corresponding to decreasing $t$. Here, $\optdm$ is replaced by the trained \textcolor{black}{score-based model} $\dm$. Deep learning approximation errors affect the accuracy of this iterative scheme, which is studied in Sec.~\ref{sec:R!0}. We formally describe our iteration algorithm by the following definition.




\begin{definition}[Discrete-time flow]
    A \emph{discrete-time flow} $\flow[\ell][][\tau_i][\theta]:\mathbb{N}\times\mathbb{R}^{|s|}\rightarrow\mathbb{R}^{|s|}$ conditioned on local history $\tau_i$ and \textcolor{black}{denoiser network} parameters $\theta$ is defined by $\flow[\ell][][\tau_i][\theta](y) = \dm[\tau_i][\flow[\ell-1][][\tau_i][\theta](y)]$, with the initial condition $\flow[0][][\tau_i][\theta](y)=y$.
\end{definition}

Intuitively, $\flow[\ell][]$ generates a denoised state after applying the denoiser network for $\ell$ times, transporting a noisy state $y$ to $\ite{y}{\ell}=\flow[\ell][][\tau_i][\theta](y)$. The distribution of these denoised states is given by the push-forward equation defined as follows.

\begin{definition}[Push-forward equation]\label{def:pushop}
The distribution of estimated states after applying $\dm$ for $\ell$ times is $p_\ell = \pusheq[\ell][][\tau_i][\theta] p_0$, where $p_0$ is the distribution of noisy states, and the push-forward operator $*$ is defined by $\pusheq[\ell][] p_0(y) = p_0(\phi_\ell^{-1}(y))$ $\det \left[\partial \phi_\ell^{-1}/\partial y\right].$
\end{definition}

In this way, the estimated posterior state distribution given by the diffusion process is ($\ell\rightarrow\infty$ indicates $T\rightarrow\infty$):
    \begin{align}
    \estpost[s][\tau_i] = \pusheq[\ell\rightarrow\infty][][\tau_i][\theta] p_0(s)\label{equ:approx_ps}.
\end{align}
According to \cref{def:pushop}, this distribution depends on the Jacobian $\partial \phi_\ell^{-1}/\partial y = (\partial \phi_\ell/\partial y)^{-1}$~\cite{lipman2022flow}. As $\phi_\ell$ is dependent on $\dm[\tau_i][y]$, our analysis would heavily utilize \textcolor{black}{denoiser network} Jacobian 
\begin{align}
    \jaco[f][][y][\tau_i] = \nabla_y\dm[\tau_i][y] =\left. \partial f/\partial y \right|_{(\tau_i, y)}.
\end{align}
$\jaco[f][][y][\tau_i]$ has eigenvalues $\eigenvalue[k][][y][\tau_i]$ and eigenvectors $\eigenvector[k][][y][\tau_i], k\in[|s|]$. Dependencies on $y$ and $\tau_i$ will be omitted in these notations when they are unambiguous within the given context. An important property we use in this paper is the \emph{Jacobian rank}, defined as the rank of the matrix \(
    \jaco[][+][y][\tau]=\left(I-\tfrac{\partial f}{\partial y}(\tau,y)\right)^{-1} \tfrac{\partial f}{\partial \tau}(\tau,y).
\)


\section{States As Shared Fixed Points}\label{sec:R=0}
We now present our findings on how diffusion models represent the one-to-many mapping from histories to states. We first consider the case with minimal influence of deep learning approximation errors in this section, and study more complex scenarios in \cref{sec:R!0}.

\subsection{Example}
We start with a didactic example. Sensor networks are a classic problem in the multi-agent literature~\cite{modi2001dynamic,nair2005networked,zhang2011coordinated} inspired by real-world challenges~\cite{lesser2003distributed}. The environment consists of multiple sensor agents and moving targets. Each agent can scan at most one nearby area per timestep, and two agents must scan an area simultaneously to track a target. Since we are studying the case with minimal influence of deep learning approximation errors in this section, we use a small sensor network with 2$\times$2 sensor agents (1\textsuperscript{st} column of \cref{fig:R0}, with each sensor represented by a circle) and 1 target. There are four possible states, each represented by a one-hot vector indicating the target's true location.

\noindent \textbf{Collectively observable (CO) Dec-POMDPs}. The first row of \cref{fig:R0} illustrates a CO Dec-POMDP~\cite{pynadath2002communicative}. Each agent's observation $o\in\mathbb{R}^2$ includes a separate dimension for each nearby area, with a value 1 if the target is present and 0 otherwise. In this example, the target is in \spname{Area\ 1}. The right side plots changes in the first two state dimensions during diffusion. Each arrow starts from a possible noisy state, which, together with local history, is the input to a \textcolor{black}{denoiser network}. The network outputs a denoised state, marking the endpoint of the arrow. \spname{Agent\ 3} and \spname{4} have uncertainty because they cannot observe the target in their nearby areas. This uncertainty is reflected in the vector fields. Conditioned on their histories (length=1), there are two fixed points, each representing a possible state. For example, \spname{Agent\ 3} cannot distinguish whether the target is in \spname{Area\ 1} or \spname{Area\ 4}; correspondingly, there are two attractors $y=(1,0,0,0)$ and $(0,0,0,1)$. Moreover, we observe that the flow has an equal probability of converging to these two fixed points, matching the true posterior state distribution given the history. \spname{Agent\ 1} and \spname{2} know the true state because they observe the target. Correspondingly, there is only one fixed point $(1,0,0,0)$ given their history. \emph{More importantly, we note that the only common fixed point shared by all agents is the true state $(1,0,0,0)$.}

\noindent \textbf{Non-collectively observable Dec-POMDPs}.\ \ The second row of \cref{fig:R0} illustrates a non-CO Dec-POMDPs. Agent observations are the same, but when the target is in \spname{Area\ 1} or \spname{2}, nearby sensors fail to observe it with 50\% probability. For example, when the true state is $(1,0,0,0)$, the observation of \spname{Agent\ 1} can be either $(1,0)$ or $(0,0)$ with equal probability. This environment is non-CO because it is possible for all agents observe $(0,0)$, in which case even aggregating all local information does not reveal the target's location, as shown in \cref{fig:R0}. We observe a significant difference from the first row: there are multiple shared fixed points: $(1,0,0,0)$ and $(0,1,0,0)$. This reflects the best inference the agents can achieve: \spname{Agent\ 3} and \spname{4} is sure that the target is not in \spname{Area\ 3} or \spname{4}. However, the aggregated local information cannot tell whether the target is in \spname{Area\ 1} or \spname{2}.

\begin{figure*}[t]
    \centering
    \vspace{-2em}
    \includegraphics[width=\linewidth]{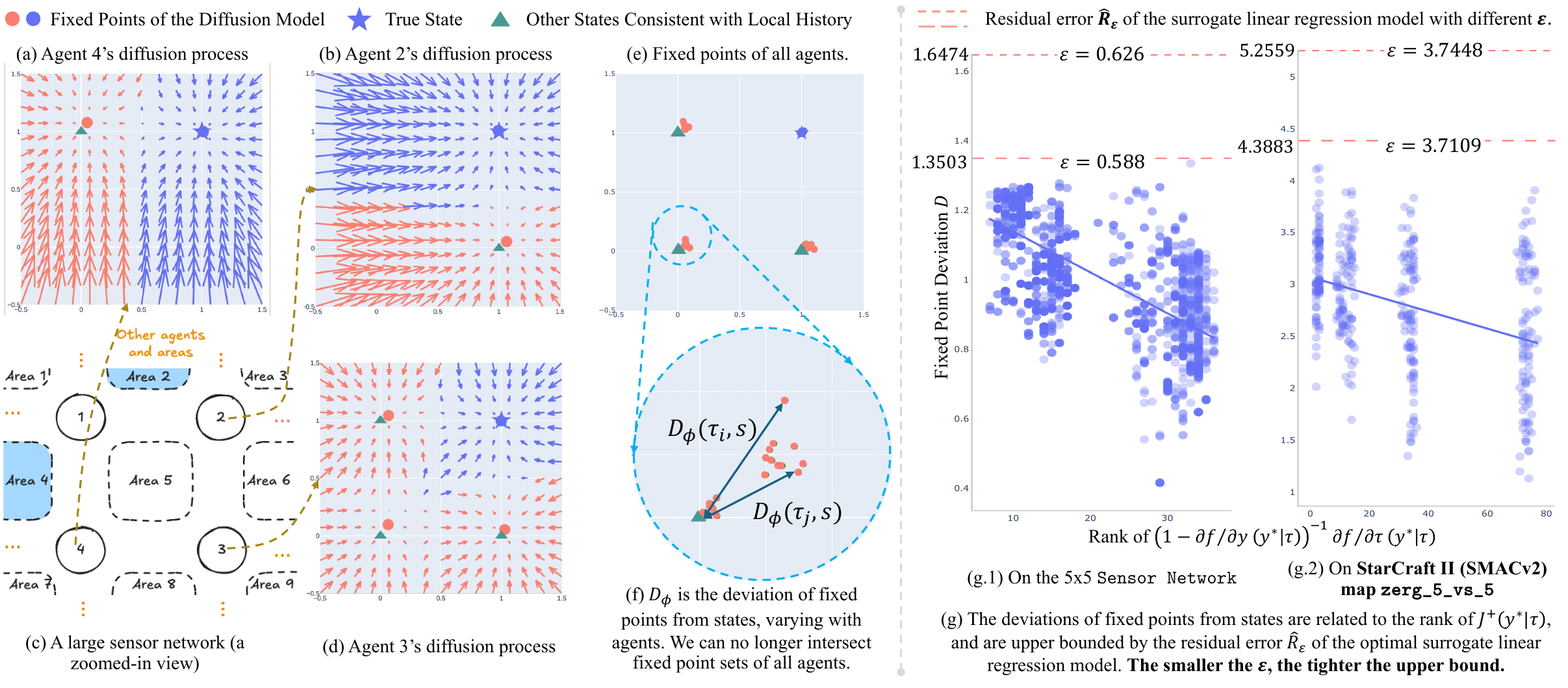}
    \vspace{-2.5em}
    \caption{Deep learning approximation errors cause fixed points to deviate from true states. Deviation norms are related to the Jacobian rank and can be upper bounded by a surrogate linear model. (a,b,d) In the 5$\times$5 sensor network (with a zoomed-in view in (c)), we show changes in state dimensions corresponding to \spname{Area 2} and \spname{4} during diffusion. (e,f) The impact of these deviations becomes evident when fixed points of all agents are displayed together in a single panel: the true state can no longer be determined by intersecting fixed point sets of all agents, as it is possible that $\cap_i \fpset[\phi][][\tau_i]=\varnothing$. (g) Empirical evidence from SMACv2 and the 5$\times$5 sensor network shows that deviation norms negatively correlate to Jacobian ranks and are tightly upper bounded by optimal residual errors of the surrogate linear model.\label{fig:R!0demo}}
    \vspace{-1em}
\end{figure*}
\subsection{Infer State Locally: Stable Fixed Points}
Although simple, the sensor network example encapsulates our findings which will be discussed in this section. Our first finding pertains to how diffusion models represent states, \ie, how each agent infers states based on its own history. We begin by showing that these individual diffusion processes converge.

\begin{restatable}{theorem}{converge}[Converged Diffusion]\label{thm:converge}
In the absence of approximation errors, repeatedly applying the \textcolor{black}{denoiser network} $\dm[\tau_i][y]$ converges to a state $s$ that is consistent with $\tau_i$ and has a dominate posterior probability given $y$: $\flow[\infty][][\tau_i][\theta](y) = s = \argmax_{s\in S_G(\tau_i)} p(s|y,\tau_i).$
\end{restatable}
\noindent \cref{thm:converge} formally underpins the observation in \cref{fig:R0}. Based on this, we give the sufficient and necessary conditions of how a diffusion model represents states.

\begin{restatable}{theorem}{repr}[State Representation]\label{thm:repr}
The diffusion model represents states consistent with $\tau_i$ by the \emph{attractors} of its flow:
    \begin{align}
        \hat{S}_G(\tau_i) = \left\{y^* \mid \pusheq[\infty][][\tau_i][\theta]p_0(y^*)>0\right\},
    \end{align}
which are equivalent to the stable fixed points $\fpset[\phi][][\tau_i]$ of \dm[\tau_i][\cdot], 
\begin{align}
    \fpset[\phi][][\tau_i] = \{y^* \mid y^*=\dm[\tau_i][y^*], |\eigenvalue[\max][][y^*][\tau_i]|<1\}.
\end{align}
Here, $\eigenvalue[\max][][y^*][\tau_i]$ is the largest eigenvalue of the Jacobian $\jaco[f][][y^*][\tau_i]$.
\end{restatable}
\noindent Proved in \cref{appx:math:theorem:fp}, \cref{thm:converge,thm:repr} align with recent research~\cite{pidstrigach2022score,chen2023sampling} showing that diffusion models are able to produce samples from data distributions with bounded support on a low-dimensional data manifold. Since \cref{thm:repr} shows \(\fpset[\phi][][\tau_i] = \hat{S}_G(\tau_i)\) always hold, the terms \emph{attractor} and \emph{fixed pointed} will be used interchangeably. In the absence of approximation errors, the stable fixed points are the states consistent with $\tau_i$: \(\fpset[\phi][][\tau_i] = S_G(\tau_i).\)

\subsection{Infer State Globally: Shared Fixed Points}
\cref{thm:converge,thm:repr} show how an individual diffusion model represents the inference of agent $i$ about states. In the simplest scenario described in the finding below, this local inference suffices to determine the true global state.
\begin{finding}
    If an individual diffusion model has only one stable fixed point $y^*$, agent $i$ is able to infer that $s=y^*$.
\end{finding}
A unique fixed point implies that only one state is consistent with $\tau_i$. Therefore, agent $i$ can unambiguously determine the global state, eliminating the need for communication with others.

We focus primarily on more complex scenarios where agents are uncertain about the state, \ie{}, individual diffusion models have multiple fixed points. We first show that, this uncertainty can be resolved in collectively observable Dec-POMDPs.

\begin{restatable}{theorem}{RzeroCO}\label{thm:R0CO}
    Without approximation errors, in collectively observable Dec-POMDPs, the intersection of the fixed point sets of all agents is the true state $s$:
    \(\cap_i \fpset[\phi][][\tau_i] =\{s\}.\)
\end{restatable}

\noindent Conversely, non-collectively observable Dec-POMDPs are more complicated, as agents are collectively unable to uniquely determine the state. However, we can still show that diffusion models identify all states consistent with the joint history and can reproduce the posterior probability of these states.

\begin{restatable}{theorem}{RzeroNonCO}\label{thm:R0NonCO}
     Without approximation errors, in non-collectively observable Dec-POMDPs, the intersection of fixed point sets is all states consistent with joint history: $\cap_i \fpset[\phi][][\tau_i] =\{s\mid p(s|\bm\tau)>0\}$. The true posterior probability $p(s|\bm\tau)$ can be recovered with appropriate prior distributions $p(y)$ of initial noisy states $y$.
\end{restatable}

\noindent Proved in \cref{appx:math:theorem:R0}, \cref{thm:R0CO,thm:R0NonCO} establish a communication protocol--agents share their fixed points when necessary to maximally resolve uncertainty about states. A potential application of these results is to train a denoiser during centralized training and use inferred states or their distributions in decentralized execution, thereby enabling new MARL algorithms (please refer to \cref{appx:policylearning}).


\section{Deviated Fixed Points}\label{sec:R!0}

We now consider the influence of deep learning approximation errors. \textbf{(1)} We find that the major impact of these errors is that fixed points deviate from states (\cref{fig:R!0demo}(a-f)). \textbf{(2)} Theoretically, we identify Jacobian ranks as the primary factor driving this deviation (\cref{thm:lr}). Empirical results (\cref{fig:R!0demo}(g)) support this finding and pinpoint  design choices that influence Jacobian ranks (\cref{fig:R!0}). \textbf{(3)} Inspired by the low-rank property, we construct a surrogate linear regression model (\cref{thm:def:surrogate}) to bound the deviation (\cref{thm:fpdbound}). \textbf{(4)} This deviation bound helps prove the convergence of a novel composite diffusion process (\cref{sec:composite}).

\subsection{Example}
We begin with a concrete example that expands the sensor network to include $5 \times 5$ agents and 2 targets, with each sensor configured to scan 4 nearby areas. \cref{fig:R!0demo}(c) zooms in on the section of the map containing two targets. For a fair comparison against the example in the previous section, we keep the network architecture and training agenda unchanged (details in \cref{appx:emp}).

In \cref{fig:R!0demo}(a,b,d), we show the discrete-time flow induced by the diffusion models in this task. The increased problem size puts an extra burden on diffusion models. For example, in \cref{fig:R!0demo}(d), the diffusion model of \spname{Agent\ 3} converges to four fixed points (blue and red circles), which are not strictly overlapped with possible states (blue star and green triangles), indicating that fixed points deviate from clean states. For better visualization, in \cref{fig:R!0demo}(f), we put the fixed points of all agents together and zoom in on the fixed points around (0,0). It shows that diffusion models of different agents exhibit distinct fixed points with varied deviations. In this way, \emph{it is no longer practical to calculate the intersection of agents' fixed point sets to obtain the true state, as the intersection would be empty.}

\subsection{Jacobian Rank and Surrogate Linear Model}

\textbf{Empirical Observations}.\ \ To understand why diffusion models can no longer represent the consistent states $S_G(\tau_i)$ accurately, we take a closer look at the fixed points learned by the diffusion models in the $5\times 5$ sensor network (\cref{fig:R!0demo}(c)) and the \spname{zerg\_5\_vs\_5} map from the complex, highly stochastic SMACv2 benchmark~\citep{ellis2023smacv2}. Specifically, we look at the Jacobian rank (the rank of \(
    \jaco[][+][y][\tau]=\left(I-\nicefrac{\partial f}{\partial y}(\tau,y)\right)^{-1} \nicefrac{\partial f}{\partial \tau}(\tau,y)
\)) at fixed points of individual agents.

\begin{figure*}[t]
\vspace{-2em}
    \centering
    \includegraphics[width=0.9\linewidth]{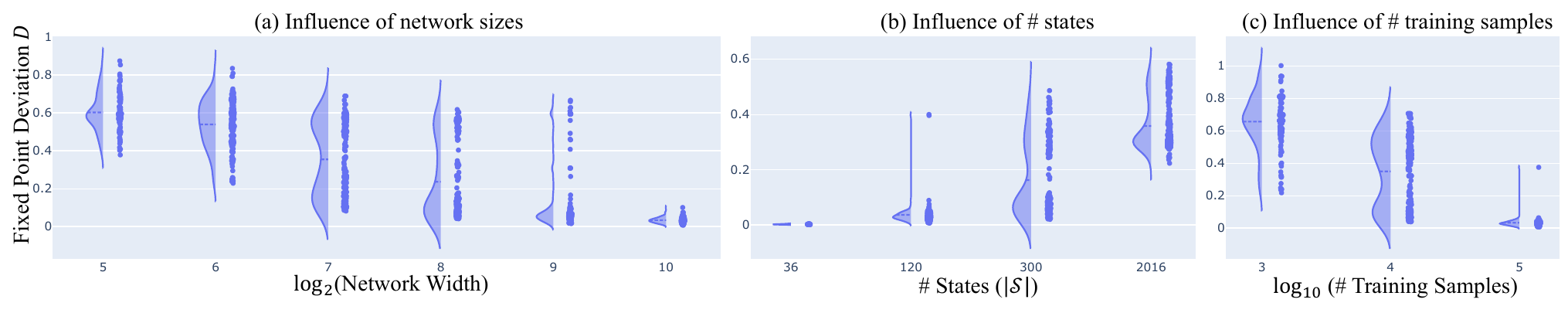}
    \vspace{-1.5em}
    \caption{Practical factors contributing to low Jacobian ranks (which correlate negatively with deviations of fixed points from true states) include narrow network architectures, large state space sizes, and small numbers of training samples. In each panel, a blue point represents the deviation of a fixed point, and distributions of these deviations are displayed on the left. }
    \label{fig:R!0}
    \vspace{-1em}
\end{figure*}

\cref{fig:R!0demo}(g) shows the relationship between fixed points' Jacobian ranks and their deviations from states (in $\ell_2$ norm). A clear negative correlation is observed between these two variables, \ie{}, fixed points with lower Jacobian ranks are more likely to exhibit large deviations from true states.

\noindent \textbf{Theoretical Understanding}.\ \  We now formally analyze the underlying reason for this negative correlation and thereby show why diffusion models might not be able to represent all states accurately. We first define the deviation of fixed points from states as follows.
\begin{definition}[Deviation of fixed points from states]\label{def:fpd}
    We define the error between a true state $s$ and its corresponding fixed point by
    \begin{align}
        \fpd[\phi][][\tau_i][s] = s-\flow[\infty][][\tau_i][\theta](s).
    \end{align}
    Here, $\flow[\infty][][\tau_i][\theta](s)$ is the attractor to which the diffusion process conditioned on $\tau_i$ converges when initialized from state $s$.
\end{definition}
Our analysis begins with the following theorem that characterizes the influence of $\tau$ on the fixed points, i.e., how the fixed point changes when $\tau$ changes.

\begin{restatable}{theorem}{lr}\label{thm:lr}
Let $y^*\in\fpset[\phi][][\tau]$ be a fixed point corresponding to history $\tau$. When $\tau$ changes to $\tau'=\tau+\Delta \tau$, the fixed point shifts to $y^{*}{}'=y^*+\Delta y^*$. If the changes in Jacobian satisfies $\|\jaco[f][][\tau',y^{*}{}']-\jaco[f][][\tau,y^{*}]\|_F<\epsilon$ for a small $\epsilon$, we have
\begin{align}
    \Delta y^* \approx \left(I-\frac{\partial f}{\partial y}(\tau,y^*)\right)^{-1} \frac{\partial f}{\partial \tau}(\tau,y^*)\Delta \tau.\label{equ:lr}
\end{align}
The approximation error in \cref{equ:lr} is bounded by $\left\| \spname{Err}(\Delta y^*) \right\| \le \frac{M\epsilon^2}{2(1-\eigenvalue[\max][])m^2}$, where $\eigenvalue[\max][]$ is the largest eigenvalue of Jacobian $\jaco[f][][y^*][\tau_i]$ at $y^*$, and $M,m$ is the upper/lower bound on the norm of Hessian. Due to its $O(\epsilon^2)$ magnitude, this error is negligible when $\epsilon$ is small.
\end{restatable}

\begin{finding}\label{finding:llbeh}
    The major takeaway of~\cref{thm:lr} emerges from \cref{equ:lr}. This equation implies that a diffusion model behaves like a (locally) linear model $\Delta y^* \approx \jaco[][+][y^*][\tau] \Delta \tau$ with weights 
\begin{align}
    \jaco[][+][y^*][\tau]=\left(I-\frac{\partial f}{\partial y}(\tau,y^*)\right)^{-1} \frac{\partial f}{\partial \tau}(\tau,y^*)\label{equ:J+}
\end{align}
to approximate the shifts of fixed points when local history changes. 
\end{finding}
In \cref{thm:lrfc}, we expand \cref{equ:lr} in the special case where the \textcolor{black}{denoiser network} is a fully-connected network.
\begin{restatable}{corollary}{lrfc}\label{thm:lrfc}
If $\dm[\tau][y]=g\left(\sigma\left((W_\tau, W_y)\begin{pmatrix} \scriptstyle\tau \\ \scriptstyle y \end{pmatrix} + b\right)\right)$is a fully connected network. $\sigma(\cdot)$ is an element-wise activation function and $g(\cdot)$ represents the subsequent fully connected layers, which may introduce additional non-linearities  following the first layer, we have $\Delta y^* \approx \eves (I-\evas)^{-1} \evas \eves[][\top] W_y^+ W_\tau \Delta \tau$, where $\eves\evas\eves[][\top]$ is eigen-decomposition of Jacobian $\jaco[f][][y^*][\tau_i]$, $W_y^+$ is Moore–Penrose inverse $W_y^+ = W_y^\top(W_yW_y^\top)^{-1}$.
\end{restatable}
\noindent Proved in \cref{appx:math:theorem:lr}, \cref{thm:lrfc} examines a specific network architecture, while the other analyses in this paper apply to any \textcolor{black}{denoiser} architecture. \cref{thm:lrfc} assumes $\jaco[f]$ is symmetric and non-negative, which is approximately true for learned denoisers~\cite{mohan2020robust} and can be proved to hold for the optimal denoiser~\cite{kadkhodaie2023generalization}. 


Based on \cref{thm:lr}, we use proof by contradiction to show that diffusion models might not have enough capacity to represent all states accurately. We start with a local history $\tau$ and one of its consistent states $s$, assuming that the diffusion model conditioned on $\tau$ has enough capacity to exactly represent $s$ by a fixed point $y^*$. We then consider other histories near $\tau$: $\mathcal{T}_{\epsilon}=\{\tau'\mid\|\jaco[f][][\tau',y^{*}{}']-\jaco[f][][\tau,y^{*}]\|_F\le \epsilon\}$. Due to \cref{finding:llbeh}, the diffusion model is trained to (locally) solve the following optimization problem.
\begin{definition}[Surrogate Local Linear Regression Model]
For a local history $\tau$ and a consistent state $s$, let $\localdata[\epsilon][][][]\subset\mathcal{D}$ be a sample set containing $(\tau', s')$ in the training dataset $\mathcal{D}$ satisfying $\|\jaco[f][][\tau,y^{*}]-\jaco[f][][\tau',y^{*}{}']\|_F$ $\le \epsilon$. The surrogate linear regression problem is:
    \begin{align}
    \mathcal{R}_\epsilon:\ \ \ \argmin_{W\in \sR^{|s|\times |\tau|}} \sum_{(\tau',s')\in\localdata[\epsilon][][][]}\|W\Delta\tau-\Delta s\|^2\label{equ:lr_model},
\end{align}
where $\Delta\tau = \tau-\tau'$, $\Delta s=s-s'$. The residual error of the optimal solution to $\mathcal{R}_\epsilon$ is $\hat{R}_\epsilon$. The number of linearly independent $s'$ in $\localdata[\epsilon][][][]$ is $r(\localdata[\epsilon][][][]) = \dim \left( \text{span} \left\{ s' \mid (\tau', s') \in \localdata[\epsilon][][][] \right\} \right)$.
\end{definition}

\noindent Intuitively, given $\Delta \tau$, the denoiser network learns $\jaco[][+][y^*][\tau]$ to minimize the difference between $\Delta y^*$ and the groundtruth $\Delta s$. The surrogate $\mathcal{R}_\epsilon$ provides the best \emph{linear} solution to this optimization problem. The question is whether the denoiser network, locally, has enough capacity to perform better than this solution.
\begin{finding}\label{thm:def:surrogate}
    If the \textcolor{black}{denoiser network} \dm~is over-parameterized with the maximum possible rank of $\jaco[][+][y^*][\tau]$ (\cref{equ:J+}) larger than the number of linearly independent state samples in the local regression problem $\mathcal{R}_\epsilon$:
    \begin{align}
        \operatorname{rank}(\jaco[][+][y^*][\tau])\ge r(\localdata[\epsilon][][][]),
    \end{align}
    then the linear regression problem $\mathcal{R}_\epsilon$ is underdetermined, indicating that the denoiser network has enough capacity to represent the history-state mapping $S_G$.
    
    On the other hand, if the \textcolor{black}{denoiser network} is under-parameterized, leaving $\operatorname{rank}(\jaco[][+][y^*][\tau])< r(\localdata[\epsilon][][][])$, then we have an over-determined regression problem that inevitably induces residual errors, leading to deviations of fixed points from states.
\end{finding}
\begin{evidence}
    We empirically verify our findings in the 5$\times$5 sensor network. In \cref{fig:R!0}-middle, we increase the number of states, so that $r(\localdata[\epsilon][][][])$ increases. With a fixed network size, we can see that $\fpd[\phi][][\tau][s]$ increases. We also find that increasing the network width can increase the rank of $\jaco[][+][y^*][\tau]$, and correspondingly decrease $\fpd[\phi]$ as shown in \cref{fig:R!0} left. This is not trivial as the input dimension is fixed and is smaller than the network width, which means the rank of $\jaco[][+][y^*][\tau]$ is actually upper bounded by the input dimension.
\end{evidence}

\subsection{Bounded Deviations}
We now discuss how to bound the deviations $\fpd[\phi][]$.

\begin{finding}\label{finding:surrogate}
When $\epsilon$ is large, the expressivity of a diffusion model is more powerful than the surrogate regression model $\mathcal{R}_\epsilon$. This is because the \textcolor{black}{denoiser network} becomes more non-linear as the Jacobian changes significantly. In this case $\hat{R}_\epsilon$ is larger than $\fpd[\phi][]$, as $\hat{R}_\epsilon$ is the residual of a linear model, while $\fpd[\phi][]$ is the residual of a non-linear model. On the other hand, when we decrease $\epsilon$, the \textcolor{black}{denoiser network} is approaching linear and its capacity is getting close to the linear regression model, so $\fpd[\phi][]$ is getting close to $\hat{R}_\epsilon$.
\end{finding}


We formally present this finding in the following theorem.
\begin{restatable}{theorem}{fpdbound}[Upper Bounded Deviation]\label{thm:fpdbound}
Let $\mathcal{M} =\{(\num{\tau}{k}, \num{s}{k})\}_{k=1}^{K_1}$. If there exist $\tau$ and $\tau'$ in $\mathcal{M}$ where $\|\jaco[f][][\tau',y^{*}{}']-\jaco[f][][\tau,y^{*}]\|_F> \epsilon$, then for $(\tau,s)\in\mathcal{M}$, we have
\begin{align}
    \fpd[\phi][][\tau][s] < \hat{R}(\mathcal{M}) = \Tr(\Sigma_{s}) - \Tr(\Sigma_{s\tau}\Sigma^{-1}_{\tau}\Sigma_{\tau s}).
\end{align}
Here $\Sigma_s = \mathbb{E}_{s\sim\mathcal{M}}[(s -  \mathbb{E}[s])(s -  \mathbb{E}[s])^\top]$, 
cross covariance $\Sigma_{s\tau} = \mathbb{E}_{(\tau,s)\sim\mathcal{M}}[(s - \mathbb{E}[s])(\tau - \mathbb{E}[\tau])^\top]$, and $\Sigma_{s\tau}=\Sigma_{\tau s}^\top$.
$\hat{R}(\mathcal{M})$ is the residual error of the optimal linear regression model on $\mathcal{M}$.
\end{restatable}

\noindent Proved in \cref{appx:math:theorem:fpdbound}, \cref{thm:fpdbound} upper bounds fixed point deviations by constructing 
a surrogate regression problem. By decreasing the $\epsilon$, we can tighten this upper bound, which helps prove the convergence of our composite diffusion process in the next section.

\begin{evidence}
We provide empirical evidence to support \cref{finding:surrogate} and \cref{thm:fpdbound}. The experiments are conducted on the 5$\times$5 sensor network (\cref{fig:R!0demo} (g.1)) and the \spname{zerg\_5\_vs\_5} map (\cref{fig:R!0demo} (g.2)) from the complex, highly stochastic SMACv2 benchmark~\citep{ellis2023smacv2}. In (g.1), two dashed horizontal lines show the residual errors $\hat{R}_\epsilon$ of the surrogate linear regression model under two different $\epsilon$ values. When $\epsilon=0.626$, the linear regression model exhibits weaker representational capacity, with $\hat{R}_{0.626}=1.6474$ significantly larger than the deviations of the fixed points. By contrast, when we decrease $\epsilon$ to 0.588, $\hat{R}_{0.588}$ provides a tight upper bound for the deviations. The case in SMACv2 is similar, \emph{indicating that surrogate models effectively approximate local behavior of diffusion models.}
\end{evidence}


\begin{figure*}[t]
    \centering
    \vspace{-1.5em}
    \includegraphics[width=\linewidth]{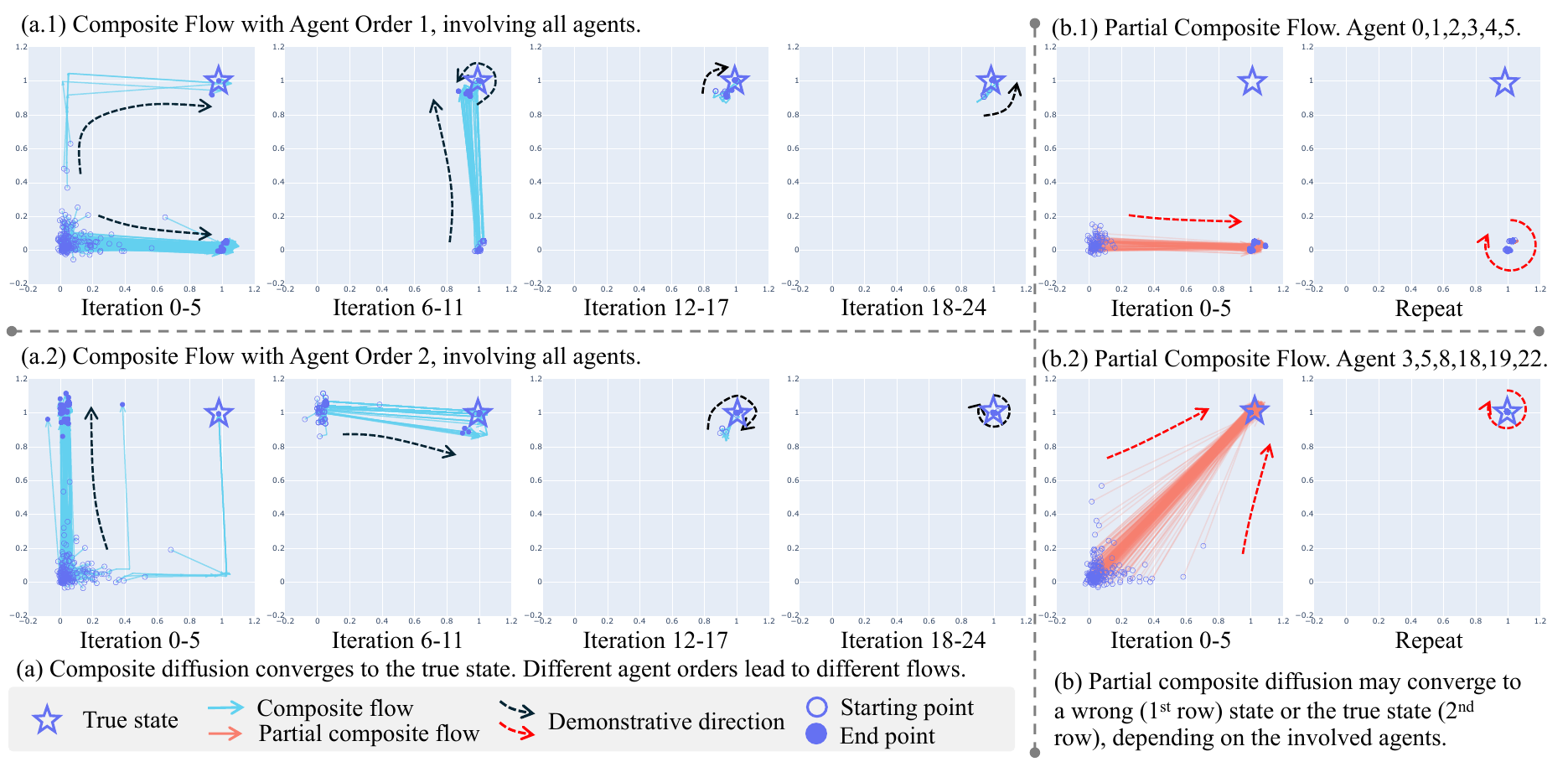}
    \vspace{-2.5em}
    \caption{Evolution of denoised state distributions (first two dimensions) during composite diffusion processes, initialized with various noisy states, in the 5$\times$5 sensor network. Each panel shows the changes (from open circles to closed circles) over 6 \textcolor{black}{denoising} iterations, with each iteration conditioned on the history of a single agent, \eg, in iteration 0-5, six agents in the corresponding order are involved. (a) Composite diffusion converges to the true state regardless of the agent ordering. (b) Partial composite diffusion may converge to incorrect states depending on the participating agents.\label{fig:composite_flow}}
    \vspace{-1em}
\end{figure*}
\section{Composite Diffusion}\label{sec:composite}

\begin{table*}[t]
    \caption{True state estimation errors measured in PSNR. Higher PSNR values indicate lower errors. Individual diffusion exhibits a provable performance gap compared to composite diffusion.\label{tab:composite-psnr}}
    \centering
    \vspace{-1em}
    \setlength{\tabcolsep}{4pt}
    \resizebox{0.9\textwidth}{!}{
    \begin{tabular}{CRCRCRCRCRCRCRCRCRCRCRCRCRCRCRCR}
        \toprule
        \multicolumn{2}{c}{\multirow{2}{*}{Setting}} &
        \multicolumn{2}{l}{\multirow{2}{*}{Alg.}} &
        \multicolumn{6}{c}{Network Width} &
        \multicolumn{6}{c}{\# Training Samples} &
        \multicolumn{6}{c}{Obs. History Length} &
        \multicolumn{4}{c}{Obs. Sight Range} &
        \multicolumn{6}{c}{Tasks} \\
        
        \cmidrule(lr){5-10}
        \cmidrule(lr){11-16}
        \cmidrule(lr){17-22}
        \cmidrule(lr){23-26}
        \cmidrule(lr){27-32}
        \multicolumn{2}{c}{} &
        \multicolumn{2}{l}{} &
        \multicolumn{2}{c}{1024} 
        & \multicolumn{2}{c}{4096} 
        & \multicolumn{2}{c}{8192} 
        & \multicolumn{2}{c}{10} 
        & \multicolumn{2}{c}{100K} 
        & \multicolumn{2}{c}{500K} 
        & \multicolumn{2}{c}{1} 
        & \multicolumn{2}{c}{5} 
        & \multicolumn{2}{c}{10} 
        & \multicolumn{2}{c}{5} 
        & \multicolumn{2}{c}{9} 
        & \multicolumn{2}{c}{Zerg} 
        & \multicolumn{2}{c}{Terran} 
        & \multicolumn{2}{c}{Protoss} \\

        \cmidrule(lr){1-4}
        \cmidrule(lr){5-6}
        \cmidrule(lr){7-8}
        \cmidrule(lr){9-10}
        \cmidrule(lr){11-12}
        \cmidrule(lr){13-14}
        \cmidrule(lr){15-16}
        \cmidrule(lr){17-18}
        \cmidrule(lr){19-20}
        \cmidrule(lr){21-22}
        \cmidrule(lr){23-24}
        \cmidrule(lr){25-26}
        \cmidrule(lr){27-28}
        \cmidrule(lr){29-30}
        \cmidrule(lr){31-32}
        \multicolumn{2}{c}{\multirow{2}{*}{Train}} & \multicolumn{2}{l}{Individual} & \multicolumn{2}{c}{13.28} & 
        \multicolumn{2}{c}{18.71} & 
        \multicolumn{2}{c}{28.20} & 
        \multicolumn{2}{c}{56.45} & 
        \multicolumn{2}{c}{28.20} & 
        \multicolumn{2}{c}{35.29} & 
        \multicolumn{2}{c}{26.39} & 
        \multicolumn{2}{c}{28.20} & 
        \multicolumn{2}{c}{31.86} & 
        \multicolumn{2}{c}{26.29} & 
        \multicolumn{2}{c}{28.20} & 
        \multicolumn{2}{c}{28.20} & 
        \multicolumn{2}{c}{29.01} & 
        \multicolumn{2}{c}{27.79} \\
        \multicolumn{2}{c}{} & \multicolumn{2}{l}{ Composite} & 
        \multicolumn{2}{c}{\textbf{15.98}} & 
        \multicolumn{2}{c}{\textbf{22.18}} & 
        \multicolumn{2}{c}{\textbf{30.77}} & 
        \multicolumn{2}{c}{\textbf{56.77}} & 
        \multicolumn{2}{c}{\textbf{30.77}} & 
        \multicolumn{2}{c}{\textbf{36.39}} & 
        \multicolumn{2}{c}{\textbf{27.88}} & 
        \multicolumn{2}{c}{\textbf{30.77}} & 
        \multicolumn{2}{c}{\textbf{33.13}} & 
        \multicolumn{2}{c}{\textbf{27.92}} & 
        \multicolumn{2}{c}{\textbf{30.77}} & 
        \multicolumn{2}{c}{\textbf{30.77}} & 
        \multicolumn{2}{c}{\textbf{30.68}} & 
        \multicolumn{2}{c}{\textbf{30.20}} \\
        \cmidrule(lr){1-2}
        \cmidrule(lr){3-4}
        \cmidrule(lr){5-6}
        \cmidrule(lr){7-8}
        \cmidrule(lr){9-10}
        \cmidrule(lr){11-12}
        \cmidrule(lr){13-14}
        \cmidrule(lr){15-16}
        \cmidrule(lr){17-18}
        \cmidrule(lr){19-20}
        \cmidrule(lr){21-22}
        \cmidrule(lr){23-24}
        \cmidrule(lr){25-26}
        \cmidrule(lr){27-28}
        \cmidrule(lr){29-30}
        \cmidrule(lr){31-32}
        \multicolumn{2}{c}{\multirow{2}{*}{Test}} & \multicolumn{2}{l}{Individual} & \multicolumn{2}{c}{13.37} & 
        \multicolumn{2}{c}{15.37} & 
        \multicolumn{2}{c}{17.42} & 
        \multicolumn{2}{c}{11.32} & 
        \multicolumn{2}{c}{17.43} & 
        \multicolumn{2}{c}{23.22} & 
        \multicolumn{2}{c}{20.77} & 
        \multicolumn{2}{c}{17.43} & 
        \multicolumn{2}{c}{16.54} & 
        \multicolumn{2}{c}{15.94} & 
        \multicolumn{2}{c}{17.43} & 
        \multicolumn{2}{c}{17.42} & 
        \multicolumn{2}{c}{17.28} & 
        \multicolumn{2}{c}{17.24} \\
        \multicolumn{2}{c}{} & \multicolumn{2}{l}{Composite} & 
        \multicolumn{2}{c}{\textbf{16.07}} & 
        \multicolumn{2}{c}{\textbf{18.49}} & 
        \multicolumn{2}{c}{\textbf{20.40}} & 
        \multicolumn{2}{c}{\textbf{12.38}} & 
        \multicolumn{2}{c}{\textbf{20.40}} & 
        \multicolumn{2}{c}{\textbf{25.49}} & 
        \multicolumn{2}{c}{\textbf{23.54}} & 
        \multicolumn{2}{c}{\textbf{20.40}} & 
        \multicolumn{2}{c}{\textbf{18.94}} & 
        \multicolumn{2}{c}{\textbf{17.90}} & 
        \multicolumn{2}{c}{\textbf{20.40}} & 
        \multicolumn{2}{c}{\textbf{20.40}} & 
        \multicolumn{2}{c}{\textbf{18.59}} & 
        \multicolumn{2}{c}{\textbf{20.08}}\\
        \toprule
    \end{tabular}
    }
    \vspace{-1em}
\end{table*}

As we discussed in \cref{sec:R!0}, when fixed points deviate from clean states, it is impractical to obtain true states by intersecting all agents' fixed points. Instead, we propose to use \textit{composite diffusion}.

\begin{definition}[Composite diffusion]
Let $[i_1,i_{2},\cdots, i_n] \in \mathcal{P}([n])$ be a permutation of $n$ agents. The composite diffusion conditioned on $\tau_{i_{1:n}}$ iteratively applies individual \textcolor{black}{denoiser models} based on each agent's history: $\dm[\tau_{i_{1:n}}][y] = f_{\theta}(\tau_{i_1}, f_{\theta}(\tau_{i_{2}},\cdots f_\theta(\tau_{i_{n}}, y)))$,
inducing a \emph{discrete-time composite flow} $\flow[\ell][][\tau_{i_{1:n}}][\theta]: \sN\times\sR^{|s|}\rightarrow \sR^{|s|}$, \(\flow[\ell][][\tau_{i_{1:n}}][\theta](y) = f(\tau_{i_{(\ell)_n}}, \flow[\ell-1][][\tau_{i_{1:n}}][\theta](y) )\), where $(\ell)_n=(\ell \mod n)$, and $\flow[0][][\tau_{i_{1:n}}][\theta](y)=y$. 
\end{definition}

\noindent We then use composite diffusion to estimate true states.


\subsection{Composite Diffusion Yields True States}

\noindent \textbf{Composite diffusion algorithm}.\ \ Our algorithm has two steps. Step 1 (Composite denoising): Sample a set of Gaussian noise $\{\ite{y}{k,0}\}_{k=0}^{K_2}$ from $\mathcal{N}(0,\sigma^2I)$. Apply composite diffusion to each $\ite{y}{k,0}$, resulting in a sequence of denoised states $(\ite{y}{k,1},\cdots,\ite{y}{k,L})$.

Step 2 (Condition check): Check whether the following conditions hold. (1) For the last $n$ elements in the sequence, $\ite{y}{k,\ell},L-n<\ell\le L$, the largest eigenvalue of the denoiser Jacobian at $\ite{y}{k,\ell}$ satisfies $|\eigenvalue[\max][]|<1$. (2) Apply the individual \textcolor{black}{denoiser} conditioned on each agent's history repeatedly to the last element $\ite{y}{k,L}$ until converge. The change in $\ite{y}{k,L}$ should be smaller than $2\fpd[\phi]$ (bounded by \cref{thm:fpdbound}). These condition checks guarantee convergence to the true state as proven in the following theorem and corollary.




\begin{restatable}{theorem}{algoconv}\label{thm:algo-conv}[Composite Diffusion Approaches True State]
    In collectively observable Dec-POMDPs, for a sequence $k$ satisfying the conditions in Step 2, the composite diffusion algorithm approaches the true state $s$ with an error bound $\max_{1\le i\le n} D_\phi(\tau_i, s)$.
    Specifically, it converges to the convex hull of the agents' fixed points near $s$. 

\end{restatable}




\begin{restatable}{corollary}{algoconvnon}\label{thm:algo-conv-non}
    In non-collectively observable Dec-POMDPs, with sufficiently enough initial samples, the composite diffusion algorithm approaches every possible global state $s$ consistent with joint history $\tau_{i_{1:n}}$ with an upper error bound $\max_{1\le i\le n} D_\phi(\tau_i, s)$.
\end{restatable}

\noindent \cref{thm:algo-conv} and \cref{thm:algo-conv-non} are proved in \cref{appx:math:theorem:algo-conv}. They highlight the advantages of composite diffusion summarized as follows.

\begin{finding}
Unlike individual diffusion conditioned on the history of a single agent, composite diffusion can resolve the uncertainty when there are multiple fixed points. Even in the simple case with only one fixed point, composite diffusion can better estimate the true state, as proved in the following theorem (details in~\cref{appx:math:theorem:accucom}).

\end{finding}

\begin{restatable}{theorem}{accucom}\label{thm:accu-com}
When there is only one fixed point for the diffusion model conditioned on $\tau_i, i\in[n]$, assume that the final element $\ite{y}{L}$ distributes uniformly, composite diffusion provides a more accurate global state estimation than individual diffusion:
\begin{align}
    \E_{\ite{y}{L}}[\|\ite{y}{L}-s\|] \le \frac{1}{|\gF_s|}\sum\nolimits_{y_i^*\in \gF_s} \|y_i^* - s\|.
\end{align}
\end{restatable}

\begin{evidence}
We evaluate the accuracy of composite diffusion (measured in peak signal-to-noise ratio, PSNR,~\citep{wang2004image}) against individual diffusion on SMACv2 \citep{ellis2023smacv2}. Training data is collected by running MAPPO \citep{yu2022surprising} (see \cref{appx:emp}). For a fair comparison, both methods use the same number of denoising steps. Composite diffusion achieves higher PSNR (indicating lower errors) across all test cases, which examine various factors that can influence diffusion processes. This gap is provable (\cref{thm:accu-com}) even when individual diffusion adopts more powerful network architectures like in \citep{xu2024beyond}.

\end{evidence}

\noindent When agents share common information (\eg, a portion of a state visible to all agents~\cite{nayyar2013decentralized}), applying composite diffusion only to private information can reduce its overhead. We next discuss overhead reduction in general cases.

\subsection{Partial Composite Diffusion}
Composite diffusion requires a communication chain in which each agent receives the output of the preceding agent's \textcolor{black}{denoiser network} and sends its own denoising output to the subsequent agent. These messages reside in $\mathbb{R}^{|s|}$. In very large systems, it is possible to trade off this communication overhead against state estimation accuracy by involving only a subset of agents in composite diffusion. 
We thereby define \textit{partial} composite diffusion $\dm[\tau_{i_{1:k}}][\cdot]$ and \textit{partial} composite flow $\flow[\ell][][\tau_{i_{1:k}}][\theta]$, in which $k<n$ and $[i_1,i_2,\cdots,i_k] \in \sP([k])$ is a permutation of the considered $k$ agents. We analyze its convergence property in \cref{thm:algo-conv-partial}.


\setcounter{theorem}{7}
\setcounter{corollary}{1}
\begin{restatable}{corollary}{algoconvpartial}\label{thm:algo-conv-partial}
    With sufficient initial samples, the partial composite diffusion algorithm approaches every possible global state $s$ consistent with $\tau_{i_{1:k}}, k<n$, with an upper error bound $\max_{1\le t \le k} D_\phi(\tau_{i_t}, s)$.
\end{restatable}
\setcounter{theorem}{8}

\begin{finding}
Composite diffusion $\dm[\tau_{i_{1:n}}][y]$ converges to the true global state, while partial composite diffusion $\dm[\tau_{i_{1:k}}][y], k<n$ may converge to wrong states, depending on the participating agents.
\end{finding}

\begin{evidence}
\cref{fig:composite_flow} shows the evolution of denoised state distributions (focusing on the first two dimensions) during (partial) composite diffusion processes in the 5$\times$5 sensor network. Initial states are sampled from $\mathcal{N}(0,I)$. Composite flows reliably discover the true state regardless of the agent ordering. In contrast, a partial composite flow stabilizes at a state consistent with participants' histories, with the accuracy depending on the participating agents.



\end{evidence}

\section{Closing Remarks}


This paper provides the first rigorous understanding of how deep learning models and their approximation errors can impact agents’ handling of PO in Dec-POMDPs. We expect that this work can establish a general framework for addressing the challenges posed by PO across various multi-agent sub-fields. As an initial demonstration of these possibilities, \cref{appx:policylearning} provides an example where integrating diffusion models with policy learning can further enhance the performance of multi-agent RL algorithms, such as MAPPO~\cite{yu2022surprising}.



\begin{acks}
We sincerely thank Stefano Ermon, Dheeraj Nagaraj, Lingkai Kong, and the anonymous reviewers of AAMAS 2025 for their valuable inputs, which significantly helped improve the quality of this paper.
\end{acks}



\balance 
\bibliographystyle{ACM-Reference-Format} 
\bibliography{sample}


\newpage
\appendix
\section{Mathematical Derivations}\label{appx:math}

\subsection{Miyasawa Relationships}\label{appx:math-miyasawa}

A clear connection between the score function and the MMSE estimator for a signal corrupted by additive Gaussian noise was initially established in~\cite{miyasawa1961empirical} and later extended in~\cite{raphan2011least,efron2011tweedie}. We provide a derivation for the case conditioned on local history.
\begin{align}
    \nabla \log p(y|\tau_i) &=\int p(s|\tau_i)p(y|s,\tau_i)\nabla_y \log p(y|s,\tau_i) ds \Big/p(y|\tau_i)\nonumber \\
    &= \int p(s|y,\tau_i)\nabla_y \log p(y|s,\tau_i) ds \\
    &= \mathbb{E}\left[\nabla_y \log p(y|s,\tau_i)\big| y,\tau_i\right]
\end{align}

Since $y\sim \mathcal{N}(s,\sigma^2I)$, we know 
\begin{align}
    \log p(y|s,\tau_i)=-\frac{1}{2\sigma^2}\|y-s\|^2-\log\left(\sqrt{2\pi\sigma^2}\right),
\end{align}
\begin{align}
    \nabla_y\log p(y|s,\tau_i)=-\frac{1}{\sigma^2}(y-s).
\end{align}
Therefore,
\begin{align}
    \nabla \log p(y|\tau_i) &= \mathbb{E}\left[-\frac{1}{\sigma^2}(y-s)\big| y,\tau_i\right] \\
    &= \frac{1}{\sigma^2}\left(\mathbb{E}_s\left[s\big| y,\tau_i\right]-y\right).
\end{align}
The optimal \textcolor{black}{denoising network} satisfies
\begin{align}
    \optdm[\tau_i][y]=y+\sigma^2 \nabla \log p(y|\tau_i) = \mathbb{E}_s\left[s\big| y,\tau_i\right].
\end{align}

\subsection{Converged Individual Diffusion Processes}\label{appx:math:theorem:fp}
We put the proofs of the following two theorems together.
\converge*
\repr*

\begin{proof}
We have shown that the diffusion model is approximating the expectation $\mathbb{E}\left[s\big| y,\tau_i\right]$. When there are no approximation errors, we have

\begin{align}
\dm[\tau_i][y] = \mathbb{E}\left[s\big| y,\tau_i\right]= \int_s s\cdot p(s | \tau_i, y)ds,
\end{align}
where
\begin{align}
p(s | \tau_i, y) = \frac{p(y | s,\tau_i) p(s|\tau_i)}{\int_s p(y | s,\tau_i) p(s|\tau_i)ds} = \frac{w_s(y)}{\int_s w_s(y)ds}
\end{align}
and
\begin{align}
w_s(y) = p(s|\tau_i) \exp\left( -\frac{\| y - s \|^2}{2\sigma^2} \right)
\end{align}

\noindent We now try to express \( \dm[\tau_i][y] \) as a gradient ascent step. Define the function
\begin{align}
L(y) = \log \left( \int_{s} w_s(y)ds \right).
\end{align}
Compute the gradient of \( L(y) \)

\begin{align}
\nabla_y L(y) &= \frac{1}{\int_s w_s(y)ds} \int_s w_s(y) \left( -\frac{(y - s)}{\sigma^2} \right) \\
&= -\frac{1}{\sigma^2}\int_s(y - s) \cdot \frac{w_s(y)}{\int_s w_s(y)ds}ds
\end{align}
It follows that

\begin{align}
\sigma^2 \nabla_y L(y) = -\int_s (y - s) \cdot p(s | \tau_i, y)ds
\end{align}
Rewriting \( \dm[\tau_i][y] \):

\begin{align}
\dm[\tau_i][y] &= \int_s s \cdot p(s | \tau_i, y)ds\\
&= y + \int_s (s - y) \cdot p(s | \tau_i, y)ds \\
&= y + \sigma^2 \nabla_y L(y)
\end{align}

It is obvious that the update rule of repeatedly applying the \textcolor{black}{denoising network} \( \ite{y}{k+1} = \dm[\tau_i][\ite{y}{k}]) \) is:

\begin{align}
\ite{y}{k+1} = \ite{y}{k} + \sigma^2 \nabla_y L(\ite{y}{k}) 
\end{align}
This is a gradient ascent step on the function \( L(y) \) with a fixed step size \( \sigma^2 \).

The function  \( L(y) \) has the following properties:

(1) \( L(y) \) is infinitely differentiable, as it is composed of exponential and logarithmic functions of smooth arguments. (2) Each term \( w_s(y) \) is log-concave in \( y \) because it contains the exponential of a negative quadratic form. Therefore, \( L(y) \) is the logarithm of a sum of log-concave functions, which is itself log-concave. (3) When $\sigma$ is small, function \( L(y) \) attains local maxima at \( y = s \), since the terms \( w_s(y) \) are maximized when \( y = s \).

Under standard assumptions for gradient ascent on smooth, concave functions, the sequence \( \{ y_k \} \) converges to a critical point (local maximum) of \( L(y) \).

The critical points of \( L(y) \) occur where \( \nabla_y L(y) = 0 \). Setting the gradient to zero:

\begin{align}
      \nabla_y L(y) = 0 \implies \int_s (y - s) \cdot p(s | \tau_i,y)ds = 0
\end{align}
This simplifies to:

\begin{align}
  y = \int_s s \cdot p(s | \tau_i,y)ds = \dm[\tau_i][y]
\end{align}

So, at critical points, \( y = \dm[\tau_i][y] \), meaning \( y \) is a fixed point of the iteration. 

We also note that, since each step is a gradient ascent, \( L(y_{k+1}) \geq L(y_k) \). \( L(y) \) is also bounded above because \( \int_s w_s(y) \) is finite. By the monotone convergence theorem, \( L(y_k) \) converges to some finite value \( L^* \), and the sequence cannot oscillate between different \( s \) because \( L(y_k) \) strictly increases unless \( \ite{y}{k} \) is already at a critical point. 


We also note that the convergence is generally exponential near the maximum due to the negative definiteness of the Hessian at \( y = s \).

When the noise level $\sigma$ is small, the limit point \( y^* \) must be one of the \( s \), as these are the only critical points where \( \nabla_y L(y) = 0 \). The \textcolor{black}{denoising network} is trained across all noise levels with $\sigma > 0$. When small values of $\sigma$ influence the diffusion process near \( s \), \( s \) are the only fixed points of the diffusion model.


\end{proof}

\subsection{Shared Fixed Points}\label{appx:math:theorem:R0}
\RzeroCO*
\begin{proof}

    When the diffusion model is accurate, each agent includes the state \( s \) within its fixed point set. We then demonstrate that the intersection of all agents' fixed point sets contains at most one element. Suppose, for the sake of contradiction, that the agents share two distinct states as fixed points, that is: \(\Large\cap_i \fpset[\phi][][\tau_i] = \{s, s'\}\). Since the diffusion models are conditioned on $\tau_i$, this implies that each agent \( i \) observes \( \tau_i \) in both states \( s \) and \( s' \). Consequently, every agent \( i \) has identical history in states \( s \) and \( s' \). However, if the Dec-POMDP is collectively observable, identical history for all agents would necessitate that \( s = s' \). This contradiction confirms that the intersection set contains at most one element.
\end{proof}

\RzeroNonCO*
\begin{proof}
    We first prove that \(\Large\cap_i \fpset[\phi][][\tau_i] =\{s\mid p(s|\bm\tau)>0\}\). For each $s$ such that $p(s|\bm\tau)>0$, we know that the global state could be $s$ when the joint history is $\bm\tau$. Therefore, $(\tau_i, s)$ is in the training dataset for the \textcolor{black}{denoising network} conditioned on each $\tau_i$. When there is no approximation errors, $s$ would be a fixed point of each diffusion model $\dm[\tau_i][\cdot]$. On the other hand, for each state $s$ that is a fixed point for all $\tau_i$, we know that the sample $(\tau_i, s)$ is in the dataset used to train the \textcolor{black}{denoising network}, indicating that $p(\tau_i|s)>0$ and thus $p(s|\bm\tau)\propto p(\bm\tau|s)p(s)>0$.

    We then look at the posterior probability of a state $\estpost[s][\tau_i]$. In the diffusion model, the probability we will reach state $s$ with $\dm[\tau_i][\cdot]$ is
    \begin{align}
        \estpost[s][\tau_i]=\int_{y\in\mathcal{B}_\phi(s)} p(y) dy.
    \end{align}
    Here $\mathcal{B}_\phi(s)$ is the basin of attraction of state $s$. From the proof of \cref{thm:repr}, we know that
    \begin{align}
        \mathcal{B}_\phi(s)=\{y \mid p(s)p(y|s)>p(s')p(y|s'),\forall s'\ne s\}.
    \end{align}
    Starting from any point $y$ in $\mathcal{B}_\phi(s)$, the diffusion model conditioned on $\tau_i$ will converge to $s$ when $\sigma$ is small. 
    
    Here and in the following proof, we abuse the notations and omit the dependence on $\tau_i$ as it is clear in the context.

    Our goal is to bound the gap between the convergence probability to state \( s \) and the prior \( p(s) \):
    \[
        \text{Gap}(s) = \left| P_{\text{converge}}(s) - p(s) \right|,
    \]
    where
    \[
    P_{\text{converge}}(s) = \int_{\mathcal{B}_\phi(s)} p(y) dy.
    \]

    The convergence probability to state \( s \) is:
    \[
    P_{\text{converge}}(s) = \int_{\mathcal{B}_\phi(s)} p(y) dy = \int_{\mathcal{B}_\phi(s)} \sum_{s'} p(s') p(y|s') dy.
    \]
    
    This can be split into two parts: (1) Correct classification (from state \( s \)):
  \[
  A(s) = \int_{\mathcal{B}_\phi(s)} p(y|s) dy,
  \]
  which is the probability that \( y \) sampled from \( p(y|s) \) falls into \( \mathcal{B}_\phi(s) \).
  (2) Misclassification (from other states \( s' \ne s \)):
  \[
  B(s, s') = \int_{\mathcal{B}_\phi(s)} p(y|s') dy,
  \]
  which is the probability that \( y \) sampled from \( p(y|s') \) is incorrectly classified into \( \mathcal{B}_\phi(s) \).

    Thus, the convergence probability is:
    \[
    P_{\text{converge}}(s) = p(s) A(s) + \sum_{s' \ne s} p(s') B(s, s').
    \]

    The gap between \( P_{\text{converge}}(s) \) and the prior \( p(s) \) is:
    \[
    \begin{aligned}
    \text{Gap}(s) &= \left| P_{\text{converge}}(s) - p(s) \right| \\
    &= \left| p(s) A(s) + \sum_{s' \ne s} p(s') B(s, s') - p(s) \times 1 \right| \\
    &= \left| p(s) (A(s) - 1) + \sum_{s' \ne s} p(s') B(s, s') \right| \\
    &= \left| -p(s) (1 - A(s)) + \sum_{s' \ne s} p(s') B(s, s') \right| \\
    &\leq p(s) (1 - A(s)) + \sum_{s' \ne s} p(s') B(s, s'),
    \end{aligned}
    \]
    since all terms are non-negative.

    We need to compute \( 1 - A(s) \) and \( B(s, s') \), which involve misclassification probabilities. We consider the decision boundary between states \( s \) and \( s' \). A sample \( y \) is misclassified from \( s \) to \( s' \) if:
    \[
    p(s) p(y|s) \leq p(s') p(y|s').
    \]
    Substituting the Gaussian distributions
    \[
    p(s) \exp\left( -\frac{\| y - s \|^2}{2\sigma^2} \right) \leq p(s') \exp\left( -\frac{\| y - s' \|^2}{2\sigma^2} \right).
    \]
    Taking the natural logarithm
    \[
    \ln p(s) - \frac{\| y - s \|^2}{2\sigma^2} \leq \ln p(s') - \frac{\| y - s' \|^2}{2\sigma^2}.
    \]
    Rewriting
    \[
    \frac{\| y - s' \|^2 - \| y - s \|^2}{2\sigma^2} \leq \ln p(s') - \ln p(s).
    \]
    Compute the difference between the squared norms
    \[
    \| y - s' \|^2 - \| y - s \|^2 = 2(y - s)^\top (s - s') + \| s - s' \|^2.
    \]
    Substitute back
    \[
    \frac{2(y - s)^\top (s - s') + \| s - s' \|^2}{2\sigma^2} \leq \ln p(s') - \ln p(s).
    \]
    Simplify
    \[
    \frac{(y - s)^\top \delta}{\sigma^2} + \frac{\| \delta \|^2}{2\sigma^2} \leq \ln p(s') - \ln p(s),
    \]
    where \( \delta = s' - s \). Rearranged
    \[
    \frac{(y - s)^\top \delta}{\sigma^2} \leq \ln p(s') - \ln p(s) - \frac{\| \delta \|^2}{2\sigma^2}.
    \]
    Define \( \Delta(s, s') = \ln p(s') - \ln p(s) \), so:
    \[
    \frac{(y - s)^\top \delta}{\sigma^2} \leq \Delta(s, s') - \frac{\| \delta \|^2}{2\sigma^2}.
    \]
    Misclassification from \( s \) to \( s' \) occurs when:
    \[
    (y - s)^\top \delta \leq \sigma^2 \left( \Delta(s, s') - \frac{\| \delta \|^2}{2\sigma^2} \right).
    \]
    Similarly, misclassification from \( s' \) to \( s \) occurs when:
    \[
    (y - s')^\top (-\delta) \leq \sigma^2 \left( \Delta(s', s) - \frac{\| \delta \|^2}{2\sigma^2} \right).
    \]
    Since \( \Delta(s', s) = -\Delta(s, s') \), we can unify the expression.

    Let \( y \sim p(y|s) \), so \( y = s + \sigma z \), where \( z \sim \mathcal{N}(0, I) \). Then
    \[
    (y - s)^\top \delta = \sigma z^\top \delta.
    \]
    Define the random variable
    \[
    X = \sigma z^\top \delta = \sigma \delta^\top z,
    \]
    since \( z \sim \mathcal{N}(0, I) \), \( \delta^\top z \sim \mathcal{N}(0, \| \delta \|^2) \). Therefore, \( X \sim \mathcal{N}(0, \sigma^2 \| \delta \|^2) \).
    The misclassification probability from \( s \) to \( s' \) is
    \[
    P_{\text{misclass}}(s \rightarrow s') = P\left( X \leq \sigma^2 \left( \Delta(s, s') - \frac{\| \delta \|^2}{2\sigma^2} \right) \right).
    \]
    Simplify the threshold
    \[
    \sigma^2 \left( \Delta(s, s') - \frac{\| \delta \|^2}{2\sigma^2} \right) = \sigma^2 \Delta(s, s') - \frac{\| \delta \|^2}{2}.
    \]
    Thus, the misclassification probability is
    \[
    P_{\text{misclass}}(s \rightarrow s') = P\left( X \leq \sigma^2 \Delta(s, s') - \frac{\| \delta \|^2}{2} \right).
    \]
    Standardize \( X \) to a standard normal variable:
    \[
    Z = \frac{ X }{ \sigma \| \delta \| } \sim \mathcal{N}(0, 1).
    \]
    Compute the standardized threshold:
    \[
    x = \frac{ \sigma^2 \Delta(s, s') - \frac{\| \delta \|^2}{2} }{ \sigma \| \delta \| } = \frac{ \sigma \Delta(s, s') }{ \| \delta \| } - \frac{ \| \delta \| }{ 2\sigma }.
    \]
    Therefore
    \[
    P_{\text{misclass}}(s \rightarrow s') = P\left( Z \leq x \right).
    \]

    Similarly, for misclassification from \( s' \) to \( s \), the threshold becomes
\[
x' = \frac{ \sigma^2 \Delta(s', s) - \frac{\| \delta \|^2}{2} }{ \sigma \| \delta \| } = -\frac{ \sigma \Delta(s, s') }{ \| \delta \| } - \frac{ \| \delta \| }{ 2\sigma }.
\]

Since the Gaussian cumulative distribution function (CDF) satisfies
\[
P(Z \leq -a) = 1 - P(Z \leq a),
\]
we can consider both \( x \) and \( x' \) for positive \( a \). However, to obtain a bound, we can use the fact that for \( a > 0 \):
\[
P(Z \leq -a) = \Phi(-a) \leq \frac{1}{a \sqrt{2\pi}} e^{-a^2/2},
\]
where \( \Phi \) is the standard normal CDF.

Assuming \( \Delta(s, s') \leq 0 \) (i.e., \( p(s') \geq p(s) \)), then \( x \leq - \frac{ \| \delta \| }{ 2\sigma } \). Thus
\[
P_{\text{misclass}}(s \rightarrow s') \leq \frac{ 2\sigma }{ \| \delta \| \sqrt{2\pi} } e^{ - \left( \frac{ \| \delta \| }{ 2\sigma } \right )^2 }.
\]

Similarly, if \( \Delta(s, s') \geq 0 \) (i.e., \( p(s) \geq p(s') \)), the misclassification probability from \( s' \) to \( s \) is bounded by the same expression.

Now, we can bound the terms in the gap expression. Misclassification from \( s \) to \( s' \):
  \[
  p(s) (1 - A(s)) = p(s) \sum_{s' \ne s} P_{\text{misclass}}(s \rightarrow s').
  \]
  Misclassification from \( s' \) to \( s \):
  \[
  \sum_{s' \ne s} p(s') B(s, s') = \sum_{s' \ne s} p(s') P_{\text{misclass}}(s' \rightarrow s).
  \]

Using the bound for each misclassification probability
\[
P_{\text{misclass}}(s \rightarrow s') \leq \frac{ 2\sigma }{ \| s - s' \| \sqrt{2\pi} } e^{ - \left( \frac{ \| s - s' \| }{ 2\sigma } \right )^2 }.
\]

Therefore, the total gap is bounded by
\[
\begin{aligned}
\text{Gap}(s) &\leq p(s) \sum_{s' \ne s} P_{\text{misclass}}(s \rightarrow s') + \sum_{s' \ne s} p(s') P_{\text{misclass}}(s' \rightarrow s) \\
&\leq \sum_{s' \ne s} \left[ p(s) + p(s') \right] \frac{ 2\sigma }{ \| s - s' \| \sqrt{2\pi} } e^{ - \left( \frac{ \| s - s' \| }{ 2\sigma } \right )^2 }.
\end{aligned}
\]
Or we can write the integral if $s'$ is from a continuous space.

\textbf{Analysis}: The bound decreases exponentially with the squared distance between states \( s \) and \( s' \) relative to the variance \( \sigma^2 \). If the states are well-separated (large \( \| s - s' \| \)) compared to the standard deviation \( \sigma \), the misclassification probabilities become negligible. Moreover, a smaller \( \sigma \) (less noise) reduces the overlap between the Gaussian distributions \( p(y|s) \), leading to a smaller gap. In these cases, we reproduce the distribution $p(s|\tau_i)$, from which $p(s|\bm\tau)$ can be reconstructed.

\end{proof}

\subsection{Jacobian Rank and Surrogate}\label{appx:math:theorem:lr}
\lr*
\begin{proof}

$y^{*}{}'$ satisfies
\begin{align}
    y^{*}{}'&= \dm[\tau'][y^{*}{}'].
\end{align}
So we have
\begin{align}
    y^*+\Delta y^* &= \dm[\tau+\Delta \tau][y^*+\Delta y^*] \\
    &\approx \dm[\tau][y^*] + \nabla_\tau \dm[\tau][y^*]\Delta \tau + \nabla_y \dm[\tau][y^*]\Delta y^* \label{equ:taylor_tau'y'}\\
    &= y^* + \jaco[f][][\tau][y^*]\Delta \tau + \jaco[f][][y^*][\tau]\Delta y^*.
\end{align}
\cref{equ:taylor_tau'y'} ignores the higher-order terms. We now bound the error introduced in this approximation.


We use the Mean Value Theorem to write:
\begin{align}
\nabla \dm[\tau'][y^{*}{}'] - \nabla \dm[\tau][y^*] = \nabla^2 \dm[\xi_\tau][\xi_y] \begin{pmatrix} \Delta \tau \\ \Delta y^* \end{pmatrix},
\end{align}
for some \( (\xi_\tau, \xi_y) \) lying between \( (\tau, y^*) \) and \( (\tau', y^{*}{}') \). Therefore,
\begin{align}
    &\left\| \nabla \dm[\tau'][y^{*}{}'] - \nabla \dm[\tau][y^*] \right\|_F \\
    = &\left\| \nabla^2 \dm[\xi_\tau][\xi_y] \begin{pmatrix} \Delta \tau \\ \Delta y^* \end{pmatrix} \right\|_F \geq m \left\| \begin{pmatrix} \Delta \tau \\ \Delta y^* \end{pmatrix} \right\|
\end{align}
where $m$ is a lower bound on the norm of the second derivative \( \nabla^2 \dm \) in the region of interest. Given that 
\begin{align}
    \left\| \nabla \dm[\tau'][y^{*}{}'] - \nabla \dm[\tau][y^*] \right\|_F < \epsilon,
\end{align}
we can write:
\begin{align}
\epsilon \geq m \left\| \begin{pmatrix} \Delta \tau \\ \Delta y^* \end{pmatrix} \right\|    
\end{align}
Therefore,
\begin{align}
\left\| \begin{pmatrix} \Delta \tau \\ \Delta y^* \end{pmatrix} \right\| \leq \frac{\epsilon}{m}
\end{align}

The second-order error (using the mean-value form of the remainder) is:
\begin{align}
    E = \frac{1}{2} \begin{pmatrix} \Delta \tau \\ \Delta y^* \end{pmatrix}^\top \nabla^2 f(\xi) \begin{pmatrix} \Delta \tau \\ \Delta y^* \end{pmatrix}
\end{align}
for some $\xi$ between \( (\tau, y^*) \) and \( (\tau', y^{*}{}') \). It follows that
\begin{align}
\left\| E \right\|\leq \frac{1}{2} \| \nabla^2 f(\xi) \| \left\| \begin{pmatrix} \Delta \tau \\ \Delta y^* \end{pmatrix} \right\|^2 \leq \frac{1}{2} M \left\| \begin{pmatrix} \Delta \tau \\ \Delta y^* \end{pmatrix} \right\|^2 \leq \frac{M\epsilon^2}{2m^2}.
\end{align}
where $M$ is an upper bound on the norm of the second derivative.

Taking $E$ into consideration, we have
\begin{align}
    y^*+\Delta y^* &= \dm[\tau+\Delta \tau][y^*+\Delta y^*] \\
    &= \dm[\tau][y^*] + \nabla_\tau \dm[\tau][y^*]\Delta \tau + \nabla_y \dm[\tau][y^*]\Delta y^* +E\\
    &= y^* + \jaco[f][][\tau][y^*]\Delta \tau + \jaco[f][][y^*][\tau]\Delta y^*+E.
\end{align}
It follows that
\begin{align}
    \Delta y^* = \left(I-\frac{\partial f}{\partial y}(\tau,y^*)\right)^{-1} \frac{\partial f}{\partial \tau}(\tau,y^*)\Delta \tau +\left(I-\frac{\partial f}{\partial y}(\tau,y^*)\right)^{-1} E.
\end{align}
When we use the approximation
\begin{align}
    \Delta y^* \approx \left(I-\frac{\partial f}{\partial y}(\tau,y^*)\right)^{-1} \frac{\partial f}{\partial \tau}(\tau,y^*)\Delta \tau 
\end{align}
the error in \( \Delta y^* \) is:
\begin{align}
\left\| \spname{Err}(\Delta y^*) \right\| \leq K \left\| E \right\| \leq K \frac{M\epsilon^2}{2m^2} =\frac{M\epsilon^2}{2(1-\eigenvalue[\max][][\tau][y^*])m^2}.
\end{align}
Here $K$ is the maximum eigenvalue of $\left(I-\frac{\partial f}{\partial y}(\tau,y^*)\right)^{-1}$, which is
\begin{align}
    K = \frac{1}{1-\eigenvalue[\max][][\tau][y^*]}.
\end{align}

\end{proof}

\lrfc*
\begin{proof}
When $\dm[\tau][y]=g(\sigma(W_\tau \tau + W_y y + b))$, where $\sigma(\cdot)$ is an element-wise activation function and $g(\cdot)$ represents the fully connect layers following the first layer. We can directly expand the derivatives
\begin{align}
    \nabla_y \dm[\tau][y^*] &= \frac{d g}{d h}\operatorname{diag}(\sigma'(h))W_y = \jaco[f][][y^*][\tau],\\
    \nabla_\tau f(\tau,y^*) &= \frac{d g}{d h}\operatorname{diag}(\sigma'(h))W_\tau,
\end{align}
where
\begin{align}
    h = W_\tau \tau + W_y y^* + b.
\end{align}

Let $W_y^+$ denote $W_y$'s Moore–Penrose inverse
\begin{align}
    W_y^+ = W_y^\top(W_yW_y^\top)^{-1}
\end{align}
we have
\begin{align}
    \nabla_\tau f(\tau,y^*) = \jaco[f][][y^*][\tau] W_y^+ W_o.
\end{align}

Thus, 
\begin{align}
    \Delta y^* &\approx (I-\jaco[f][][y^*][\tau])^{-1} \jaco[f][][y^*][\tau] W_y^+ W_\tau \Delta \tau \\
    &= \left(\eves(I-\evas)^{-1} \eves[][\top] \right) \eves\evas \eves[][\top] W_y^+ W_\tau \Delta \tau \\
    &= \eves (I-\evas)^{-1} \evas \eves[][\top] W_y^+ W_\tau \Delta \tau.
\end{align}

It also follows that
\begin{align}
    \|\Delta y^*\|^2 &\approx \sum_k \left(\frac{\eigenvalue[k]}{1-\eigenvalue[k]} \right)^2 \left\langle \eigenvector[k], W_y^+ W_\tau \Delta \tau\right\rangle^2.
\end{align}
\end{proof}

\subsection{An Upper Bound of Deviations}\label{appx:math:theorem:fpdbound}
\fpdbound*

\begin{proof}
    When $\|\jaco[f][][\tau',y^{*}{}']-\jaco[f][][\tau,y^{*}]\|_F> \epsilon$, where $\epsilon$ is a constant, according to the proof of \cref{thm:lr}, the error term $\spname{Err}(\Delta y^*)$ in the approximation
    \begin{align}
    \Delta y^* \approx \left(I-\frac{\partial f}{\partial y}(\tau,y^*)\right)^{-1} \frac{\partial f}{\partial \tau}(\tau,y^*)\Delta \tau.
    \end{align}
    is no longer negligible. Specifically, for each $(\tau, s)$ pair, the linear approximation is correct by
    \begin{align}
    \spname{Err}(\Delta y^*) = \frac{1}{2} \begin{pmatrix} \Delta \tau \\ \Delta y^* \end{pmatrix}^\top \nabla^2 f(\xi) \begin{pmatrix} \Delta \tau \\ \Delta y^* \end{pmatrix}
    \end{align}
    The Hessian matrix $\nabla^2 f(\xi)$ is trained to lower the MSE training loss. In this way, the diffusion model is more powerful than the surrogate linear regression model and thus has a lower error. 
\end{proof}

\subsection{Composite Diffusion Algorithm Details}\label{appx:algo}

\noindent \textbf{Composite diffusion algorithm}.\ \ We provide a detailed description of our composite diffusion algorithm. Our algorithm has two steps. Step 1 (Composite denoising): Sample a set of Gaussian noise $\{\ite{y}{k,0}\}_{k=0}^{K_2}$ from $\mathcal{N}(0,\sigma^2I)$. Apply composite diffusion to each $\ite{y}{k,0}$, resulting in a sequence of denoised states $(\ite{y}{k,1},\cdots,\ite{y}{k,L})$. 

Step 2 (Condition check): Check whether the following two conditions hold. (1) Eigenvalue condition. For the last $n$ elements in the sequence, $\ite{y}{k,\ell},L-n<\ell\le L$, the largest eigenvalue of the diffusion Jacobian at $\ite{y}{k,\ell}$ satisfies $|\eigenvalue[\max][]|<1$. (2) Distance condition.  Apply individual \textcolor{black}{denoising network} conditioned on each agent's history $\tau_i$ repeatedly to the last element $\ite{y}{k,L}$ until converge, the change in $\ite{y}{k,L}$ should be smaller than $2\fpd[\phi]$ (using the upper bound in \cref{thm:fpdbound}). 
We apply individual flow on $\ite{y}{k,L}$ and obtain $n$ fixed points: $y_i^* = \flow[\infty][][\tau_{i}][\theta](\ite{y}{k,L})$. To check whether $\ite{y}{k,L}$ lies in the vicinity of the shared state, the distance between $y_i^*$ and $\ite{y}{k,L}$ should be small: $\max_{1\le i \le n}\|y_i^* - \ite{y}{k,L}\| < 2D_\phi$, where $D_{\phi} = \max_{\tau,s}\fpd[\phi][][\tau][s]$ is the maximal error. If these two conditions are not met, we simply discard this sample and proceed to another $\ite{y}{k,0}$. These condition checks guarantee convergence to the true state as proven in the following theorem.

\subsection{Composite Diffusion}\label{appx:math:theorem:algo-conv}
\algoconv*


\begin{proof}

We first prove the existence of $y$ such that both convergence conditions of our diffusion algorithm could be met and then prove the convergence behavior of different cases. 

To allow effective composite diffusion, we require the network to have low approximation errors, \ie{}, the distance from the state to its corresponding fixed point is small, and the fixed point is stable. Specifically, the distance from the state to its corresponding fixed point is small: $D_{\phi}< \frac{1}{6} \delta_s$, where $D_{\phi} = \max_{\tau,s}\fpd[\phi][][\tau][s]$ is the maximal error and $\delta_s = \min_{s,s'\in \gS}\|s-s'\|$ is the minimal state distance. The fixed point is stable, \ie{}, the eigenvalue of the Jacobian of the \textcolor{black}{denoising network} $\dm[\tau_{i}][y], \forall i\in [1,2,\cdots,n]$ near its fixed point $y_i^*$ should be small: $\forall y \in \{y \mid \|y-y_i^*\| < 2 D_{\phi} \}$, we have $|\eigenvalue[k][][y][\tau_i]| < 1,  \forall k\in [|s|]$.

In the following proof, we omit the superscript $k$ for simplicity. We denote the start noisy state by $\ite{y}{0}$ and the end denoised state by $\ite{y}{\infty}$. 

\textbf{Existence proof.} We first prove the existence of $\ite{y}{\infty}$ such that the eigenvalue condition and distance condition can be met. 

Let $\gF_s = \{y_1^{*}(s), y_2^*(s), \cdots, y_n^*(s)\}$ be the fixed points of all agents, obtained from the true global state $s$: $y_i^*(s)=\flow[\infty][][\tau_{i}][\theta](s)$. Let $B(y_i^*(s), D_{\phi}) = \{y| \|y- y_i^*(s) < 2 D_{\phi} \|\}$ be the ball centering $y_i^*(s)$ with radius $D_{\phi}$. Since $D_{\phi} = \max_{\tau,s}\fpd[\phi][][\tau][s]$ is the maximal approximation error, we have $B_s = \bigcap_{i=1}^n B(y_i^*(s), D_{\phi}) \ne \emptyset$. Since the fixed point $y_i^*(s)$ is stable, we have that $\forall y\in B_s, |\eigenvalue[k][][y][\tau_i]| < 1,  \forall k\in [|s|]$. That is, the eigenvalue condition is met. Then according to the definition of $\fpd[\phi][][\tau][s]$ (\cref{def:fpd}) and $D_{\phi}$, the distance condition is met.
Therefore, $\forall \ite{y}{\infty}\in B_s$, the two conditions are both met. 


\textbf{Convergence proof}. 
We next prove the convergence of the composite diffusion algorithm. We divided our discussion into two parts: $\ite{y}{0}\in B_s$ and $\ite{y}{0}\notin B_s$. 


\textit{Case 1: $\ite{y}{0}$ is initialized in $B_s$}. 

Let $\gF = \{y_1^*, y_2^*, \cdots, y_n^*\}$ be the fixed points of all agents, where $y_i^*=\flow[\infty][][\tau_{i}][\theta](\ite{y}{0})$. In this case, we always have $\gF = \gF_s$. For simplicity, we use $y_i^*$ to represent $y_i^*(s)$ in case 1.  

We now prove that for any starting noisy state $\ite{y}{0}\in B_s$, the denoised state of the composite flow $\ite{y}{\ell}=\flow[\ell][][\tau_{i_{1:n}}][\theta](\ite{y}{0})$ stays in the convex hull of fixed points near the true global state $s$ after enough iterations $\ell$, \ie{}, $\ite{y}{\ell}\in Conv(\gF) \subseteq B_s$, where
\begin{align}
Conv(\gF):=\left\{\sum_{y_i^*\in \gF}\beta_i y_i^*\mid \beta_i>0, \sum_i^n \beta_i = 1\right\}
\end{align}

When applying $\flow[1][][\tau_{i_{1:n}}][\theta]$ to $\ite{y}{\ell-1}$, we have the following transformation equations
\begin{align}
    y_i^* - y_{\ell,i} = \alpha_{\ell,i}(y_i^*-y_{\ell,i-1}), \forall i, 
\end{align}  
where $y_{\ell,0}=\ite{y}{\ell-1}$, $\ite{y}{\ell}=y_{\ell, n}$ and $\alpha_{\ell,i}=\sqrt{\sum_{k=1}^{|s|}\eigenvalue[k][][y_{\ell,i}][\tau_i]^2}$ is the scaling ratio and $|\alpha_{\ell,i}|<1, \forall i$. This equation assumes that $y$ will move towards $y_i^*$ straightforwardly when $y$ is close to $y_i^*$.

By eliminating $y_{\ell,i}$, we can expand $\ite{y}{\ell}$

\begin{align*}
    \ite{y}{\ell} =& (1-\alpha_{\ell,n})y_n^* + \alpha_{\ell,n} y_{\ell,n-1} \\
    =& (1-\alpha_{\ell,n})y_n^* + \alpha_{\ell,n}\left((1-\alpha_{\ell,n-1})y_{n-1}^* + \alpha_{\ell,n-1} y_{\ell,n-2} \right) \\
    =& (1-\alpha_{\ell,n})y_n^* + \alpha_{\ell,n}(1-\alpha_{\ell,n-1})y_{n-1}^* + \alpha_{\ell,n}\alpha_{\ell,n-1} y_{\ell,n-2} \\
    & \cdots \\
    =& \sum_i^n (1-\alpha_{\ell,i})y_i^* \left(\prod_{j=i+1}^n \alpha_{\ell,j} \right) + \ite{y}{\ell-1}\left(\prod_{i=1}^n \alpha_{\ell,i}\right).
\end{align*}

Let $w_{\ell,i}=\prod_{j=i+1}^n\alpha_{\ell,j}$ and $w_{\ell,n}=1$, we have
\begin{align}
    \ite{y}{\ell} = \sum_i^n w_{\ell,i}(1-\alpha_{\ell,i})y_i^* + w_{\ell,0}\ite{y}{\ell-1}.
\end{align}
Every coefficient of $y_i^*, \ite{y}{\ell-1}$ is non-negative and the sum of all coefficients is

\begin{align}
     &\sum_{i=1}^n w_{\ell,i}(1-\alpha_{\ell,i}) + w_{\ell,0} \\
    =&\sum_{i=2}^n w_{\ell,i}(1-\alpha_{\ell,i}) + w_{\ell,1}(1-\alpha_{\ell,1}) + w_{\ell,0} \\
    =&\sum_{i=2}^n w_{\ell,i}(1-\alpha_{\ell,i}) + w_{\ell,1}\\
    \cdots &\\
    =&\sum_{i=n}^n w_{\ell,i}(1-\alpha_{\ell,i}) + w_{\ell,n-1}\\
    =& w_{\ell,n}-w_{\ell,n}\alpha_{\ell,n} + w_{\ell,n-1} \\
    =& 1
\end{align}

Thus, $\ite{y}{\ell}\in Conv(\gF\cup \{\ite{y}{\ell-1}\})$. 
If $\ite{y}{\ell-1}\in Conv(\gF)$, we have $\ite{y}{\ell}\in Conv(\gF)$. 
Otherwise, suppose $\ite{y}{0} \notin Conv(\gF)$. Let $\bar{y}_{\ell} = \sum_i^n w_{\ell,i}(1-\alpha_{\ell,i})y_i^*/(1-w_{\ell,0})\in Conv(\gF)$, we expand $\ite{y}{\ell}$:

\begin{align*}
    \ite{y}{\ell} =& (1-w_{\ell,0}) \bar{y}_{\ell} + w_{\ell,0}\ite{y}{\ell-1} \\
    =& (1-w_{\ell,0}) \bar{y}_{\ell} + w_{\ell,0}\left((1-w_{\ell-1,0}) \bar{y}_{\ell-1} + w_{\ell-1,0}\ite{y}{\ell-2} \right) \\
    =& (1-w_{\ell,0}) \bar{y}_{\ell} + w_{\ell,0}(1-w_{\ell-1,0})\bar{y}_{\ell-1} + w_{\ell,0}w_{\ell-1,0}\ite{y}{\ell-2}\\
    &\cdots\\
    =& \sum_{i=1}^{\ell}(1-w_{i,0})\bar{y_{i}}\left(\prod_{j=i+1}^{\ell}w_{j,0}\right) + \ite{y}{0}\left(\prod_{i=1}^{\ell} w_{i,0}\right)\\
\end{align*}

Let $\xi_{\ell,i}=\prod_{j=i+1}^{\ell}w_{j,0}$ and $\xi_{\ell,\ell}=1$. We have
\begin{align}
    \ite{y}{\ell} =& \sum_{i=1}^{\ell}\xi_{\ell,i}(1-w_{i,0}) \bar{y_{i}} + \xi_{\ell,0}\ite{y}{0}.
\end{align}
Then
\begin{align}
    \lim_{\ell\to +\infty} \xi_{\ell,0} = \lim_{\ell\to +\infty} \prod_{j=1}^{\ell} w_{j,0} = \lim_{\ell\to +\infty} \prod_{j=1}^{\ell} \prod_{k=j+1}^n \alpha_{\ell,k} = 0,
\end{align}
which indicates that after enough iterations $\ell$, $\ite{y}{\ell} = \sum_{i=1}^{\ell}\xi_{\ell,i}(1-w_{i,0}) \bar{y_{i}} \in Conv(\gF)$. 
Finally, we conclude that even $\ite{y}{0} \notin Conv(\gF)$, after enough iterations $\ell$, we can achieve $\ite{y}{\ell}\in Conv(\gF)$ and then $\forall i>\ell, \ite{y}{i}\in Conv(\gF)$. 



\textit{Case 2: $\ite{y}{0}$ is not initialized in $B_s$}. 

Now we assume that the starting noisy state is far from the true global state: $\ite{y}{0}\notin B_s$. After enough iteration $\ell$, if $\ite{y}{\ell}$ moves in $B_s$, we can still use the proof of Case 1 to show that $\ite{y}{\infty} \in Conv(\gF_s)$. 

Hence, we only need to consider the cases when it does not find $B_s$: $\ite{y}{\infty} \notin B_s$. Assume $\ite{y}{\infty}$ meets the two checking conditions. 

Let $\gF = \{y_1^*, y_2^*, \cdots, y_n^*\}$ be the fixed points of all agents, where $y_i^*=\flow[\infty][][\tau_{i}][\theta](\ite{y}{\infty})$. We first show that $\gF \ne \gF_s$. If we assume that $\gF = \gF_s$. Since $\ite{y}{\infty} \notin B_s$, there must be $y_i^*$ such that $\|y_i^* - \ite{y}{\infty}\| \ge 2D_\phi$, which breaks the distance condition check and this sample should discarded. Thus, we always have $\gF \ne \gF_s$. 

Let $s_i = \argmin_{s\in \gS}\|s - y_i^*\|$ be the closest state to the fixed point $y_i^*$. There must be at least two $y_i^*, y_j^*$ such that $s_i \ne s_j$, otherwise, all $y^*_i$s share the same state that is different from $s$, which contradicts our collectively observable settings. Then
\begin{align}
& \|s_i - s_j\| \\
< & \| (s_i - y_i^*) + (y_i^* - \ite{y}{\infty}) + (\ite{y}{\infty} - y_j^*) + (y_j^* - s_j) \|  \\
< &  \|s_i - y_i^*\| + \|y_i^* - \ite{y}{\infty}\| + \|\ite{y}{\infty} - y_j^*\| + \|y_j^* - s_j\| \\
< & D_\phi + 2 D_\phi + 2 D_\phi + D_\phi \\
< & 6 D_\phi.
\end{align}
However, we know that $\|s_i - s_j\|\ge \delta_s > 6 D_\phi$ because of the low approximation error, which still contradicts our assumption. Thus $\ite{y}{\infty}$ does not meet the two checking conditions and should be discarded.


In conclusion, after enough iterations $\ell$, the composite diffusion algorithm will converge to the convex hull of the fixed points of all agents: $\ite{y}{\ell} \in Conv(\gF)$.

Since $D_\phi(\tau_i,s)=\|y_i^* - s\|$, where $y_i^*$ is the fixed point and the vertex of $Conv(\gF)$, we always have:
\begin{align}
    \|\flow[\infty][][\tau_{i_{1:n}}][\theta](\ite{y}{0}) - s\| \le \max_{1\le i\le n} D_\phi(\tau_i, s).
\end{align}

\end{proof}

\algoconvnon*

\begin{proof}
    In non-collectively observable Dec-POMDPs, the state size could be larger than 1. Specifically, let $\gS(\tau_{1:n})=\{s_j|p(s_j|\tau_{1:n}) > 0\}$ be the set of possible states for joint histories $\tau_{1:n}$. $|\gS(\tau_{1:n})| \ge 1$. 
    
    Let $\gF_{s_j}=\{y_1^*(s_j),y_2^*(s_j),\cdots,y_n^*(s_j)\}$ be the set of fixed points near state $s_j\in \gS(\tau_{1:n})$, such that $y_i^*(s_j) = \flow[\infty][][\tau_{i}][\theta](s_j)$.  

    Let $B(y_i^*(s_j), D_{\phi}) = \{y| \|y- y_i^*(s_j) < 2 D_{\phi} \|\}$ be the ball centering $y_i^*(s_j)$ with radius $D_{\phi}$. Since $D_{\phi} = \max_{\tau,s}\fpd[\phi][][\tau][s]$ is the maximal approximation error, we have $B_{s_j} = \bigcap_{i=1}^n B(y_i^*(s_j), D_{\phi}) \ne \emptyset$.

    Assume $\ite{y}{0}_j \in B_{s_j}$, with the same existence proof and convergence proof, we have that:

    \begin{align}
        \flow[\infty][][\tau_{i_{1:n}}][\theta](\ite{y}{0}_j)\in Conv(\gF_{s_j})
    \end{align}
    And therefore
    \begin{align}
        \|\flow[\infty][][\tau_{i_{1:n}}][\theta](\ite{y}{0}_j) - s\| \le \max_{1\le i\le n} D_\phi(\tau_i, s).
    \end{align}

\end{proof}


\algoconvpartial*

\begin{proof}
Consider partial composite diffusion flow $\flow[\infty][][\tau_{i_{1:k}}][\theta]$, the set of possible state for current joint history is $\gS(\tau_{i_{1:k}})=\{s_j|p(s_j|\tau_{i_{1:k}}) > 0\}$. $|\gS(\tau_{i_{1:k}})|\ge 1$ since current considered agents might not fully observe the state.

Let $\gF_{s_j}=\{y_{i_1}^*(s_j),y_{i_2}^*(s_j),\cdots,y_{i_k}^*(s_j)\}$ be the set of fixed points near state $s_j \in \gS(\tau_{i_{1:k}})$, such that $y_{i_t}^*(s_j) = \flow[\infty][][\tau_{i_t}][\theta](s_j)$. 

Let $B(y_{i_t}^*(s_j), D_{\phi}) = \{y| \|y- y_{i_t}^*(s_j) < 2 D_{\phi} \|\}$ be the ball centering $y_{i_t}^*(s_j)$ with radius $D_{\phi}$. Since $D_{\phi} = \max_{\tau,s}\fpd[\phi][][\tau][s]$ is the maximal approximation error, we have $B_{s_j} = \bigcap_{i=1}^n B(y_{i_t}^*(s_j), D_{\phi}) \ne \emptyset$.

Assume $\ite{y}{0}_j \in B_{s_j}$, with the same existence proof and convergence proof, we have that:

\begin{align}
    \flow[\infty][][\tau_{i_{1:k}}][\theta](\ite{y}{0}_j)\in Conv(\gF_{s_j})
\end{align}
And therefore
\begin{align}
    \|\flow[\infty][][\tau_{i_{1:k}}][\theta](\ite{y}{0}_j) - s\| \le \max_{1\le t \le k} D_\phi(\tau_{i_t}, s)
\end{align}

\end{proof}


\subsection{Better Accuracy of Composite Flow}\label{appx:math:theorem:accucom}
\accucom*

\begin{proof}

Since our composite diffusion algorithm moves $\ite{y}{0}$ to the convex hull of the fixed points $\gF_s$ near the true global state $s$, we assume $\ite{y}{L}\sim \text{Uniform}(Conv(\gF_s))$. 

Any point $\ite{y}{L}$ in the convex hull $Conv(\gF_s)$ can be expressed as a convex combination of the points $y_i^*\in \gF_s$:
\begin{align}
    \ite{y}{L} = \sum_{i=1}^{|\gF_s|} \beta_i y_i^*,
\end{align}
where $\beta_i\ge 0$ and $\sum_{i=1}^{|\gF_s|}\beta_i=1$. Since we are considering a uniform distribution over the convex hull, the coefficient $\beta_i$ are uniformly distributed over the simplex $\{(\beta_1,\beta_2,\cdots)|\beta_i\ge 0, \sum_{i=1}^{|\gF_s|}=1\}$.

For any $\ite{y}{L}\in Conv(\gF_s)$:
\begin{align}
    \| \ite{y}{L} - s \| = \left\| \sum_{i=1}^{|\gF_s|} \beta_i (y_i^* - s) \right\| \leq \sum_{i=1}^{|\gF_s|} \beta_i \| y_i^* - s \|.
\end{align}

Taking the expectation over $\beta$ uniformly distributed over the simplex:
\begin{align}
    \mathbb{E}_{\beta}\left[ \| \ite{y}{L} - s \| \right] \leq \mathbb{E}_{\beta}\left[ \sum_{i=1}^{|\gF_s|} \beta_i \| y_i^* - s \| \right] = \sum_{i=1}^{|\gF_s|} \mathbb{E}_{\beta}[\beta_i] \| y_i^* - s \|.
\end{align}

Since $\beta$ is uniformly distributed over the simplex, the expected value of each $\beta_i$ is $\mathbb{E}_{\beta}[\beta_i] = \frac{1}{|\gF_s|}$. This is because the simplex is symmetric with respect to all $\beta$.

Substituting $\mathbb{E}_{\beta}[\beta_i] = \frac{1}{|\gF_s|}$ into the inequality:
\begin{align}
    \mathbb{E}_{\beta}\left[ \| \ite{y}{L} - s \| \right]\leq \sum_{i=1}^{|\gF_s|} \left( \frac{1}{|\gF_s|} \| y_i^* - s \| \right) = \frac{1}{|\gF_s|}\sum_{y_i^*\in \gF_s} \|y_i^* - s\|.
\end{align}

Thus, we have

\begin{align}
    \E_{\ite{y}{L}\sim\text{Uniform}(Conv(\gF_s))}[\|\ite{y}{L}-s\|] \le \frac{1}{|\gF_s|}\sum_{y_i^*\in \gF_s} \|y_i^* - s\|.
\end{align}

\end{proof}

\section{Implementation Details}\label{appx:emp}

We detail our implementation of empirical analysis. 

\textbf{Architectures.} Our empirical experiments are performed with vector inputs. Therefore, instead of using UNet~\citep{ronneberger2015unet} or BF-CNN~\citep{mohan2019bfcnn} as in previous studies on diffusion models, we use fully-connected networks with 6 hidden layers. The default dimensions are $1024$ (Sensor Network) and $8192$ (SMACv2). For simplicity, we do not add any LayerNorms or BatchNorms. 

\textbf{Training.} We follow the training procedure described in \citet{kadkhodaie2023generalization, mohan2019bfcnn}, which minimizes the mean squared error in denoising the noisy inputs corrupted by i.i.d. Gaussian noise with standard deviations drawn from the range $[0,1]$. All trainings are carried out on batches of size 512 for 1000 epochs.  We also follow the denoiser architectures such that they are not conditioned on denoising steps or noise levels, which allows them to handle a range of noises. During the denoising process, these architectures also allow us to operate with any noise levels. We thus do not have to specify the step sizes~\citep{kadkhodaie2020solving}. 

\begin{figure}[t]
    \centering
    \includegraphics[width=\linewidth]{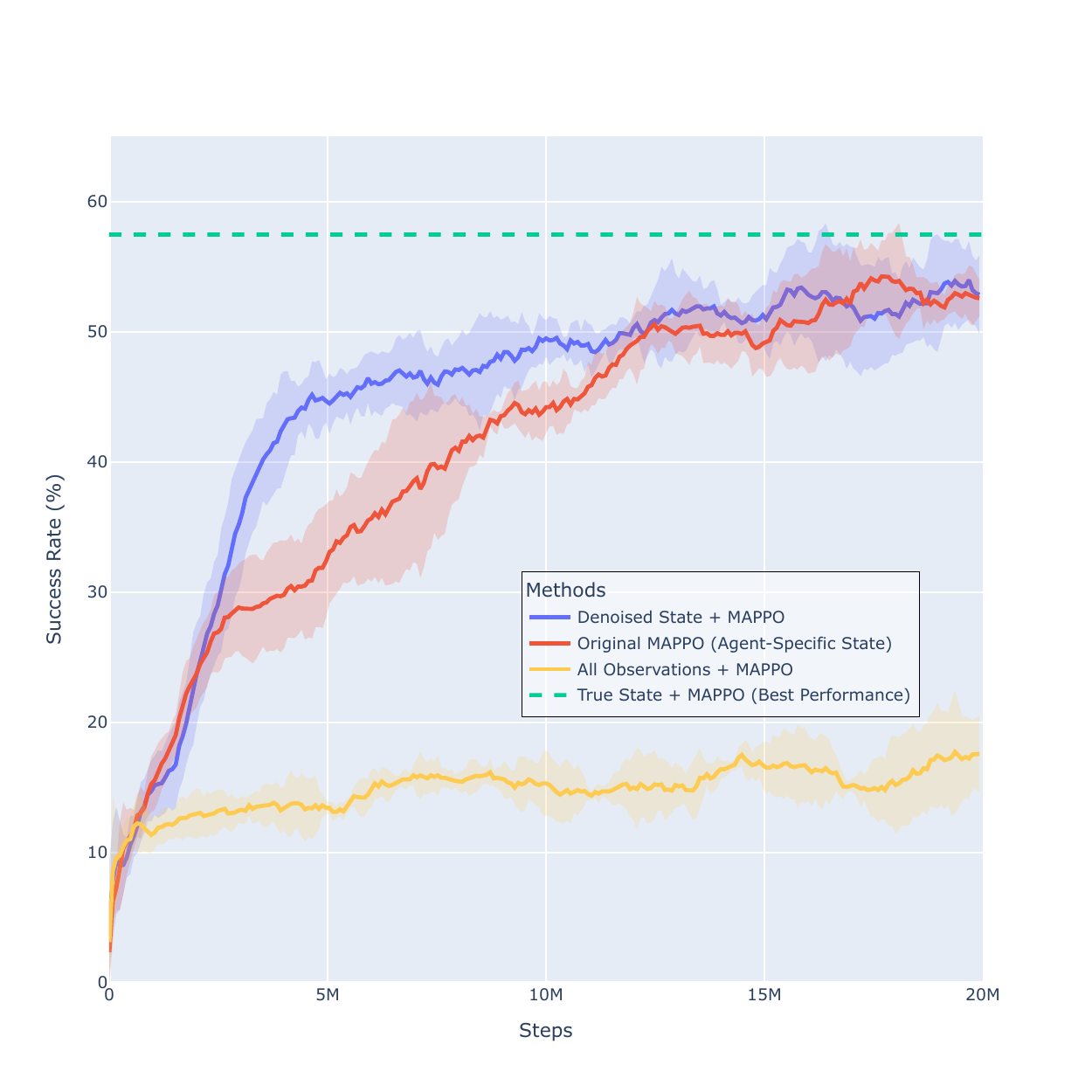}
    \caption{Policy learning performance of MAPPO with agent-specific states (as in MAPPO's vanilla implementation), joint histories, true states, and states predicted by diffusion models.}
    \label{fig:policylearning}
\end{figure}

\textbf{Datasets.} The datasets are collected from two environments: Sensor Network and SMACv2~\citep{ellis2023smacv2}. Sensor Network is a simple but illustrative environment to show the influence of partial observability in multi-agent settings. There are $n$ agents in the environments and each agent can only observe its neighboring areas. One or two targets move around and the agents' goal is to scan the target. We consider two cases: collectively observable cases and non-collectively observable cases. In the collectively observable cases, agents' observations are deterministic, and the agents together can observe the true state.  In non-collectively observable cases, agents' observations are stochastic and they may make mistakes. For example, one target is in the observable area, but the agent's observation tells no target. SMACv2 is a well-known multi-agent benchmark that focuses on decentralized micromanagement scenarios in StarCraft II. SMACv2 makes three major changes to SMACv1~\citep{samvelyan19smac}: randomizing start positions, randomizing unit types, and changing the unit sight and attack ranges. These changes bring more challenges of partial observabilities, In this work, we adopt the cutting-edge algorithm MAPPO~\citep{yu2022surprising} to collect data. We use their official hyperparameters and run for 20M environment steps. We use the data in the replay buffer, \ie{}, the observations of all agents and the global states.

\section{Example Usage of Our Method}\label{appx:policylearning}

As an example of potential applications of our method, we discuss how our results apply to the centralized training with decentralized execution (CTDE) framework.

In the training phase, we use collected history-state pairs to train a denoiser. Once trained, during execution, the denoiser is able to generate a distribution $\hat{p}$ of states consistent with local history. When $\hat{p}$ has a support size of one, an agent can individually infer the true global state without communication. This inferred state can be used, for example, as the input to the agent's local policy, with the error bounded by \cref{thm:fpdbound}. On the other hand, when $\hat{p}$ has a support size greater than one, our method enables a variety of algorithmic possibilities. Some examples include: \textbf{(1) Safe multi-agent learning.} Agents can account for the worst-case scenario by sampling $s\sim\hat{p}$ and selecting the state with the minimized value $V(s)$. By optimizing $V(s)$, we can derive a safe policy that is robust against the most adverse states. \textbf{(2) Distributional multi-agent learning.} $\hat{p}$ can also be used to generate a distribution over Q-values. Operating within this distributional space opens up possibilities analogous to those in distributional reinforcement learning.

As an initial step, we test a simple approach where we randomly sample a state from $\hat{p}$ to serve as input to the policy network. The denoiser is trained online using samples collected during training. We compare our method ($\mathtt{Denoised\ State+MAPPO}$) on the map $\mathtt{zerg\_5\_vs\_5}$ from SMACv2~\citep{ellis2023smacv2} against three baselines: (1) $\mathtt{Original\ MAPPO}$ using its official implementation, (2) $\mathtt{All\ Observations+MAPPO}$, which takes observations of all agents as policy input, and (3) $\mathtt{True\ State+MAPPO}$ that takes ground-truth states as input. The results are presented in \cref{fig:policylearning}. Compared to $\mathtt{Original\ MAPPO}$, our method can infer possible global states, which are shown to be able to accelerate policy learning. While our method slightly underperforms $\mathtt{True\ State+MAPPO}$, this is expected as we directly sample states from $\hat{p}$ which may not always be the true global state. In further work, we anticipate that the methods discussed in this section can further improve the learning performance.


\end{document}